\documentclass{article}

\usepackage{arxiv}

\usepackage[utf8]{inputenc} 
\usepackage[T1]{fontenc}    
\usepackage[hidelinks]{hyperref} 
\usepackage{url}            
\usepackage{booktabs}       
\usepackage{amsfonts}       
\usepackage{nicefrac}       
\usepackage{microtype}      
\usepackage{graphicx}
\usepackage{natbib}
\usepackage{doi}

\usepackage{xcolor}         

\usepackage{array}

\usepackage{threeparttable}
\usepackage{subcaption} 
\usepackage{caption}
\usepackage{amsthm}
\usepackage{multicol}
\usepackage{multirow}
\setcitestyle{numbers,square}

\usepackage{amsmath}
\usepackage{amssymb}
\usepackage{mathtools}
\usepackage{enumitem}
\usepackage{tabularx}

\newcommand{\best}[1]{\textbf{#1}}
\newcommand{\second}[1]{\underline{#1}}

\theoremstyle{plain}
\newtheorem{theorem}{Theorem}[section]
\newtheorem{proposition}[theorem]{Proposition}
\newtheorem{lemma}[theorem]{Lemma}
\newtheorem{corollary}[theorem]{Corollary}
\theoremstyle{definition}

\theoremstyle{remark}
\newtheorem{remark}[theorem]{Remark}

\title{TabNSM: Neural Sparse Mixer for Tabular Regression}

\author{
  Ali Eslamian \\
  Department of Computer Science \\
  University of Kentucky \\
  Lexington, KY 40506 \\
  \And
  Qiang Cheng \\
  Department of Computer Science \\
  Institute for Biomedical Informatics \\
  University of Kentucky \\
  Lexington, KY 40506 \\
}

\date{}

\renewcommand{\headeright}{}
\newcommand{\undertitle}{}
\renewcommand{\shorttitle}{TabNSM}

\hypersetup{
pdftitle={TabNSM: Neural Sparse Mixer for Tabular Regression},
pdfauthor={Ali Eslamian, Qiang Cheng},
}

\begin{document}

\maketitle

\begin{abstract}
Large-scale, high-dimensional tabular regression remains challenging. Tree-based models are often robust but do not support end-to-end representation learning, while deep models offer flexible feature learning but can incur costly interaction modeling and remain sensitive to noisy or redundant features. We propose \emph{TabNSM}, a scalable regression framework that extends our earlier sparse-attention and mixer architectures. TabNSM adapts these components within an \emph{Adaptive Sparse Interaction Module (ASIM)}, which combines foreground feature discovery, sparse local interaction encoding, and Feature-Token Mixing with near-linear complexity under fixed sparse configurations. The regression-specific contributions are a Multi-Stage Regression Head, \emph{GridLoss}, an ordinal-aware soft-binning objective, and \emph{RISE} (Reweighted Instance Sampling by Error), a difficulty-aware sampling strategy based on loss-quantile bins. Across nine real-world regression benchmarks, TabNSM achieves strong predictive performance and practical scalability, with consistent gains on high-dimensional and heterogeneous datasets.
\end{abstract}

\section{Introduction}
Modern applications in healthcare, finance, and industrial monitoring depend on large-scale tabular regression to forecast continuous outcomes, such as clinical risk scores or sensor-based targets. 
These tasks involve high-dimensional, heterogeneous feature spaces comprising numerical signals, categorical attributes, and text-derived embeddings~\citep{DANets}. A defining characteristic of such data is that predictive signals are rarely distributed evenly across the feature space; instead, a small subset of features may be decision-relevant for each instance, while the majority are noisy, redundant, or weakly informative~\citep{eslamian2025tabnsa}. This sparse and instance-specific structure contrasts with image and text modalities, where informative content is organized into coherent spatial or sequential regions.

This structure motivates a foreground-background view of high-dimensional tabular regression. For each sample, the feature space can be viewed as containing a small \emph{foreground} of decision-relevant features and a much larger \emph{background} of uninformative or noisy features. Crucially, this partition is not globally fixed: the foreground shifts across samples, so no single feature subset characterizes the entire dataset. A model that cannot dynamically isolate the foreground may conflate signal with background variation, degrading both accuracy and robustness.

The empirical strength of tree-based models on tabular data can be partly understood from this perspective. Models such as CatBoost perform recursive, axis-aligned feature selection, progressively isolating sparse predictive signals through hierarchical splitting \citep{Ngboost, NODE, prokhorenkova2018catboost}. This mechanism provides a structural prior well aligned with tabular sparsity. However, tree-based models are less suited to end-to-end differentiable representation learning, seamless integration of high-dimensional semantic embeddings, and joint optimization with learned feature encoders. These limitations become increasingly important in modern text-enriched tabular settings.

Deep tabular models offer greater representational flexibility but face complementary challenges. Transformer-style architectures typically treat tabular features as tokens and use self-attention to model cross-feature dependencies, but dense self-attention computes all pairwise interactions, resulting in quadratic cost with respect to feature dimension and potentially mixing informative features with large numbers of noisy or redundant background features. Recent state-space models (SSMs) \citep{ahamed2024mambatab, FT-Mamba, wang2024mambular} reduce this computational burden through linear-time sequential scanning, but their reliance on an imposed feature ordering does not directly provide sparse, instance-adaptive feature selection. Existing deep tabular models therefore face a persistent trade-off among expressive interaction modeling, robust foreground isolation, and computational scalability.

Building on our earlier TabNSA sparse-attention architecture and TabMixer mixing block~\citep{eslamian2025tabnsa,eslamian2025tabmixer}, we propose \textbf{Tabular Neural Sparse Mixer (TabNSM)}, a scalable framework for large-scale, high-dimensional tabular regression. TabNSM adapts these prior components within an \textit{Adaptive Sparse Interaction Module (ASIM)}, which combines instance-adaptive foreground discovery, sparse local interaction encoding, and Feature-Token Mixing (FTM). By selecting structured subsets of feature interactions, ASIM provides a differentiable analogue of tree-like selective feature isolation while preserving the expressiveness and end-to-end trainability of neural networks. FTM further promotes information propagation beyond the sparse attention support, enabling the model to balance selective feature interaction with broader feature-level context without incurring dense quadratic attention cost.

To support robust training with continuous targets, TabNSM further introduces two complementary training components. \textit{GridLoss} aligns predictions and targets in an ordinal grid space, providing a smooth, structure-aware complement to standard pointwise regression losses. \textit{RISE} ({R}eweighted {I}nstance {S}ampling by {E}rror) periodically reweights training examples according to per-sample loss, encouraging the model to focus on difficult or underfit regions of the target distribution. Together, ASIM, GridLoss, and RISE provide a unified framework for scalable and robust high-dimensional tabular regression.

In summary, the methodological contributions of this regression-focused extension are as follows:
\begin{itemize}
    \item {Regression-oriented ASIM integration:} an adaptation of the TabNSA sparse-attention pattern and TabMixer-style feature-token mixing within a unified high-dimensional regression pipeline.
    \item {Multi-Stage Regression Head:} a progressively refined prediction head with residual coupling and feature fusion for coarse-to-fine regression from the ASIM-enhanced representation.
    \item {GridLoss:} a differentiable soft-binning objective that aligns predictions and continuous targets in an ordinal grid space, complementing standard pointwise regression losses.
    \item {RISE:} a difficulty-aware sampling strategy that periodically reweights training instances using loss-quantile bins to improve robustness under non-uniform error distributions.
\end{itemize}

We evaluate TabNSM on nine real-world regression benchmarks across diverse domains and data regimes. Results show strong predictive performance and practical scalability, while ablations validate the complementary roles of ASIM, the Multi-Stage Regression Head, GridLoss, and RISE.
\section{Related Work}
\label{sec:related_work}

\textbf{Deep Learning versus Tree Ensembles for Regression.}
Gradient Boosted Decision Trees (GBDTs), including XGBoost~\citep{chen2016xgboost}, LightGBM~\citep{ke2017lightgbm}, and CatBoost~\citep{prokhorenkova2018catboost}, remain strong baselines for tabular regression and frequently outperform deep architectures on medium-scale datasets with limited feature complexity~\citep{grinsztajn2022tree,borisov2022deep}. However, GBDTs are not naturally end-to-end differentiable and are less suited to joint representation learning with high-dimensional learned features, limiting their flexibility in modern multimodal or text-enriched tabular settings. While tree-based ensembles remain competitive, their training cost can increase substantially with dataset size, number of trees, and tree depth, which may limit scalability in large-scale or high-dimensional regimes (Appendix~\ref{appendix:dimension}).

Deep tabular models have progressed beyond standard multilayer perceptrons (MLPs) to include regularization-focused architectures~\citep{kadra2023well}, logic-aware networks~\citep{abutbul2020dnf}, and hybrid tree-neural models~\citep{xu2024ncart}. Recent tabular foundation models (TFMs), notably TabPFN~\citep{hollmann2022tabpfn} and TabPFNv2~\citep{hollmann2025tabpfn}, achieve strong generalization through in-context learning. However, their transformer-based designs can incur substantial computational and memory costs due to attention over samples and/or features, often scaling quadratically in the relevant dimension. In addition, inference may require conditioning on training examples, which can limit scalability on large datasets. As a result, practical use of such models is often constrained by dataset size, feature dimensionality, and memory budget~\citep{mueller2025mothernet, TabPFNv2limitation, grinsztajn2025tabpfn}.

\textbf{Feature Interaction and Scalability.}
Transformer-based tabular models, such as TabTransformer~\citep{huang2020tabtransformer} and FT-Transformer~\citep{gorishniy2021revisiting}, enable end-to-end representation learning by using dense self-attention to capture global feature interactions. However, their dense attention mechanisms can be sensitive to feature noise in high-dimensional settings~\citep{TabNet-high,gorishniy2021revisiting}. While expressive, dense attention incurs $\mathcal{O}(D^2)$ computational and memory costs with respect to feature dimension $D$, which becomes expensive for wide tables with hundreds of features. Recent sparse attention variants~\citep{eslamian2025tabnsa, park2024optimizing} reduce this burden through block-diagonal masks, hierarchical patterns, or selective attention mechanisms. Concurrently, State-Space Models (SSMs), such as Mamba~\citep{gu2023mamba}, provide linear-time sequence mixing along the feature axis through selective scan operations, improving scalability and parameter efficiency~\citep{wang2025comprehensive}. However, SSM-based tabular models typically rely on imposed feature ordering and do not explicitly model sparse pairwise feature interactions.

\textbf{Handling Heterogeneous Data and Text Features.}
Modern tabular datasets incorporate high-dimensional textual embeddings alongside structured numerical and categorical fields~\citep{dissanayake2025distilling, TabLLM}. While preprocessing relies on ordinal, label, or one-hot encoding, recent work has shown that pretrained language models can provide semantic representations for text-rich tabular data~\citep{data2023embedding}. Integrating semantic text embeddings, such as those from Sentence-BERT, can improve performance on benchmarks containing textual attributes~\citep{hollmann2025tabpfn, choudhury2025tabard}. TabICL~\citep{TabICL} performs in-context learning at inference time but is designed for classification rather than large-scale regression. BiSHop~\citep{xu2024bishop} models column- and row-wise dependencies via bi-directional sparse Hopfield layers but scales poorly to high-dimensional datasets. MITRA~\citep{Mitra} pretrains on a mixture of synthetic priors for in-context learning but does not consistently outperform TabPFNv2 on high-dimensional regression.

\textbf{Regression Objectives and Difficulty-Aware Learning.}
Beyond standard MSE and MAE, robust losses such as Huber and quantile losses can improve resistance to outliers and heterogeneous noise~\citep{khaki2023guideline}. Distribution-aware and soft-label approaches improve regression calibration by discretizing continuous targets or representing them through label distributions~\citep{li2022rethinking, lin2024tabular}. Curriculum learning and hard-example mining are widely used in classification, but remain less developed in regression, particularly for long-tailed targets and heterogeneous error distributions. These settings motivate training strategies that emphasize difficult samples without destabilizing optimization.

\textbf{Positioning of TabNSM.}
Prior work suggests that scalable tabular regression requires more than reducing the cost of feature mixing. GBDTs provide strong feature selectivity but lack differentiable representation learning; dense-attention models learn flexible interactions but scale poorly with feature dimension; and SSM-based models improve efficiency but do not explicitly impose sparse, instance-adaptive feature interaction structure~\citep{ren2025deep, PATRO2025111279}. Our earlier TabNSA model adapted compressed-block selection and local sparse attention to tabular features, while TabMixer introduced parallel feature- and token-mixing branches~\citep{eslamian2025tabnsa,eslamian2025tabmixer}. TabNSM is a follow-on regression framework: ASIM reuses and adapts those principles rather than claiming a wholly new sparse-attention primitive. The new contribution is their regression-specific integration with a Multi-Stage Regression Head, GridLoss, and RISE. These components jointly address selective feature interaction, ordered continuous targets, and heterogeneous training difficulty.
\section{Methodology} \label{method}
\subsection{Problem Formulation and Preprocessing}
Let $\{(x_i, y_i)\}_{i=1}^{N}$ be a supervised regression dataset where $x_i \in \mathbb{R}^{D}$ represents a sample of $D$ heterogeneous features and $y_i \in \mathbb{R}$ is the continuous target. Our objective is to learn a mapping $f_\Theta : \mathbb{R}^{D} \rightarrow \mathbb{R}$ that minimizes a robust regression objective while maintaining scalability in the high-dimensional regime ($D \gg 10^2$).

To address the heterogeneity of tabular data, we unify numerical, categorical, and textual features into a dense representation. Following standard practice \citep{pedregosa2011scikit}, we categorize columns into numerical, low-cardinality categorical, and high-cardinality or free-text features.

Numerical features are standardized to zero mean and unit variance. Low-cardinality categorical variables (fewer than 10 unique values) are integer-encoded; this avoids the dimensionality expansion of one-hot encoding, while allowing downstream layers to learn category-dependent interactions. Crucially, all transformation parameters are fitted solely on the training split to prevent data leakage.

For high-cardinality and free-text fields (e.g., descriptions or addresses), we use semantic embeddings~\citep{tab2text2024, llmembeddings2025}. Text columns are concatenated per row and encoded with a frozen \texttt{all-MiniLM-L6-v2} Transformer~\citep{reimers2019sentence, wang2020minilm}. A fixed vector $E_i \in \mathbb{R}^{d_{\mathrm{text}}}$ is obtained via masked mean pooling over last-layer states $H_i$ as $E_i=\frac{\sum_t M_{i,t} H_{i,t}}{\sum_t M_{i,t}+\epsilon}$, where $M_{i,t}\in\{0,1\}$. Embeddings are $\ell_2$-normalized. The final input $x_i$ concatenates numerical, categorical, and text features. The approach is zero-shot but compatible with larger LLM encoders (e.g., Gemma) for stronger text--target alignment~\citep{TabLLM}.

\subsection{Architecture Overview}
Building on the published TabNSA--TabMixer backbone~\citep{eslamian2025tabnsa,eslamian2025tabmixer}, TabNSM makes two regression-specific architectural changes:
(i) ASIM couples the inherited structured sparse-attention and FTM pathways through learnable residual blending; and
(ii) a \emph{Multi-Stage Regression Head} applies residually coupled MLP branches followed by a compact fusion head to produce a continuous scalar prediction. Figure~\ref{fig:tabnsm_framework} summarizes the resulting architecture.
\begin{figure}[t]
  \centering
  \includegraphics[width=\linewidth]{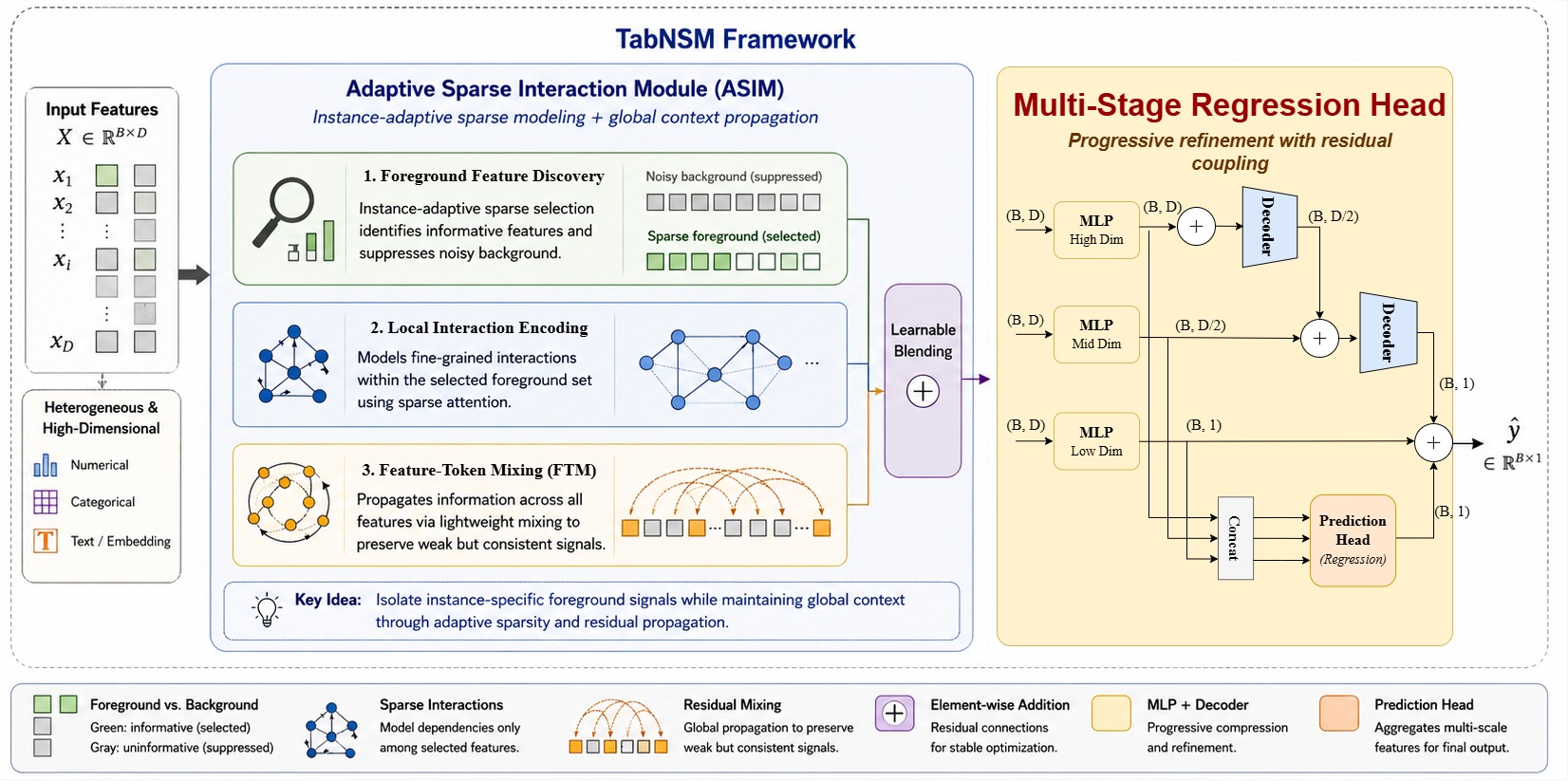}
  \caption{\textbf{Overview of the TabNSM architecture.}
TabNSM extends the published TabNSA--TabMixer representation backbone with regression-specific residual integration and a Multi-Stage Regression Head. Within ASIM, the inherited structured sparse-attention operator provides instance-adaptive feature selection, while the inherited FTM operator propagates broader feature-level context. Learnable residual blending couples these streams, and the resulting representation is passed to progressively narrowed MLP branches with residual coupling and final feature fusion to produce $\hat{y}$.}
\label{fig:tabnsm_framework}
\vspace{-10pt}
\end{figure}
\subsection{Adaptive Sparse Interaction Module (ASIM)}
To model cross-feature interactions at scale, we project each scalar feature into a $d_{\mathrm{tok}}$-dimensional embedding space. Given an input vector $x \in \mathbb{R}^{D}$, each feature $x_d$ is mapped via a learnable projection to obtain the tokenized representation $X_{\mathrm{tok}} = \operatorname{Tok}(x) \in \mathbb{R}^{D \times d_{\mathrm{tok}}}$, where $X_{\mathrm{tok}}[d,:] \in \mathbb{R}^{d_{\mathrm{tok}}}$ denotes the token vector for the $d$-th feature. (Batch dimensions are omitted for clarity).

\textbf{Instance-Adaptive Sparse Feature-wise Attention.}
Following the feature-wise adaptation introduced in TabNSA~\citep{eslamian2025tabnsa}, we use a sparse attention mechanism that treats the $D$ input features as a sequence, where each feature token serves as a query for the current sample. Foreground feature discovery is implemented through three components: (i)
a \emph{local sliding window} of size $s$ that captures local interactions under the chosen feature ordering; (ii) \emph{compressed block representations}, which aggregate blocks of size $b$ with stride $r$ into summary keys using a learnable MLP, whose resulting attention scores identify the most relevant feature groups for each query; and (iii) retrieval of the top-$m$ highest-scoring blocks at full resolution as the \emph{foreground}. Because the attention scores are computed from sample-specific query tokens, the selected blocks vary across inputs, allowing each sample to dynamically isolate its own foreground feature groups:
\begin{equation}
\tilde{X} = \mathcal{A}_{\mathrm{sparse}}\!\left(X_{\mathrm{tok}}; \, s,b,r,m,H,d_h\right) \in \mathbb{R}^{D \times d_{\mathrm{tok}}}.
\label{eq:sparse-attn}
\end{equation}
Here $H$ is the number of attention heads and $d_h$ is the per-head dimension, with $d_{\mathrm{tok}}=H d_h$. This mechanism enables instance-adaptive focus on informative feature-to-feature dependencies and can improve robustness to background noise when the sparse support excludes irrelevant dimensions (Proposition~\ref{prop:noise_short}; see Appendix~\ref{app:sparse_attn} for the full formulation).

\textbf{Feature-Token Mixing (FTM) and Residual Blending.}
To capture global dependencies that might fall outside the sparse attention mask, we adapt the dual-axis mixing design of TabMixer~\citep{eslamian2025tabmixer} as an FTM operator $\mathcal{M}(\cdot)$. It performs per-instance mixing across the embedding and feature dimensions using lightweight MLPs with parallel branches (Appendix~\ref{app:tabmixer}). To stabilize optimization, we combine the sparse and original token streams through learnable scalars $w_1,w_2,w_3\in\mathbb{R}$. All intermediate tensors share the same shape $\mathbb{R}^{D\times d_{\mathrm{tok}}}$ as $X_{\mathrm{tok}}$:
\begin{align}
U &= w_1 X_{\mathrm{tok}}+\tilde{X}, \\
Z &= \mathcal{M}(U), \\
Y &= w_2 U+Z+w_3X_{\mathrm{tok}}.
\end{align}

Finally, we aggregate across the token axis to return to the feature space via 
$y_{\mathrm{feat}}=\operatorname{mean}_{\mathrm{tok}}(Y)\in\mathbb{R}^{D}$. 
We refer to this module as the Adaptive Sparse Interaction Module (ASIM); it yields a feature-level representation enriched by adaptive selectivity and global mixing while maintaining linear memory complexity in $D$ under fixed sparse settings.

\begin{proposition}[ASIM Complexity (Informal)]
\label{prop:ASIM_complexity_brief}
Under the notation above, let the sparse attention operator attend to at most $L=L(s,b,r,m)$ keys per query feature. For fixed token dimension $d_{\mathrm{tok}}=H d_h$, number of heads $H$, and head dimension $d_h$, one ASIM layer has sparse-attention complexity $\mathcal{O}(D H d_h L)=\mathcal{O}(D d_{\mathrm{tok}} L)$ and activation memory $\mathcal{O}(D(d_{\mathrm{tok}}+L))$. 
Under fixed sparse hyperparameters, fixed token dimension, and fixed feature mini-batch size $B_f$, both compute and memory scale linearly with the feature dimension $D$.
\end{proposition}
The formal statement and proof are in Appendix~\ref{sec_proof_ASIM}.

The following proposition gives a simplified conditional analysis: if the sparse support covers the informative features while excluding most background dimensions, then sparse aggregation improves the signal-to-noise ratio relative to dense aggregation.
\begin{proposition}[Conditional SNR Improvement under Sparse Feature Selection]
\label{prop:noise_short}
Let $\mathbf{x}=[\mathbf{x}_{\mathcal{S}};\mathbf{x}_{\mathcal{N}}]\in\mathbb{R}^D$, where $\mathcal{S}$ indexes $s$ informative features with average signal $\mu$ and $\mathcal{N}$ indexes i.i.d.\ zero-mean noise features with variance $\sigma^2$. Consider dense aggregation over all $D$ features and sparse aggregation over a support $\mathcal{T}$ satisfying $\mathcal{S}\subseteq\mathcal{T}$ and $|\mathcal{T}|=mb$. Then, for $mb>s$,
\[
\mathrm{SNR}_{\mathrm{dense}}
=
\frac{s|\mu|}{\sigma\sqrt{D-s}},
\qquad
\mathrm{SNR}_{\mathrm{sparse}}
=
\frac{s|\mu|}{\sigma\sqrt{mb-s}}.
\]
Thus, when the sparse support covers the informative features while excluding most background dimensions, sparse aggregation improves SNR by
\[
\frac{\mathrm{SNR}_{\mathrm{sparse}}}{\mathrm{SNR}_{\mathrm{dense}}}
=
\sqrt{\frac{D-s}{mb-s}}.
\]
In the ideal case $mb=s$, the sparse support contains no background noise, so the noise contribution vanishes under this simplified model.
\end{proposition}
The full statement and proof are provided in Appendix~\ref{app:noise_proof}. 

\subsection{Multi-Stage Feature Refinement}
The TabNSM head processes the enriched representation $y_{\text{feat}}\in\mathbb{R}^{D}$ from the ASIM through a multi-scale residual pipeline. This design enables the model to iteratively refine predictions via coarse-to-fine pathways. 

\textbf{Progressive Branches.}
Rather than using a purely sequential bottleneck, TabNSM employs residually coupled branches that repeatedly access the full representation $y_{\mathrm{feat}}$ while progressively reducing dimensionality:
\begin{align}
y_1 &= \mathrm{Dec}_1\!\left(\mathrm{MLP}_1(y_{\mathrm{feat}}) + y_{\mathrm{feat}}\right) \in \mathbb{R}^{D/2}, \label{eq:branch1}\\
y_2 &= \mathrm{Dec}_2\!\left(\mathrm{MLP}_2(y_{\mathrm{feat}}) + y_1\right) \in \mathbb{R}^{1}, \label{eq:branch2}\\
y_3 &= \mathrm{MLP}_3(y_{\mathrm{feat}}) + y_2 \in \mathbb{R}^{1}, \label{eq:branch3}
\end{align}
where $\mathrm{MLP}_1: \mathbb{R}^D \to \mathbb{R}^D$, $\mathrm{MLP}_2: \mathbb{R}^D \to \mathbb{R}^{D/2}$, and $\mathrm{MLP}_3: \mathbb{R}^D \to \mathbb{R}^1$ are shallow networks with GELU activations. The decoders $\mathrm{Dec}_1$ and $\mathrm{Dec}_2$ facilitate dimensionality reduction. In this configuration, Eq.~\eqref{eq:branch2} performs residual learning in the intermediate $D/2$ space by combining the coarse representation $y_1$ with a secondary feature-informed projection, while Eq.~\eqref{eq:branch3} provides the final scalar refinement. This architecture ensures that even the final prediction remains grounded in the high-dimensional feature interactions encoded by $y_{\mathrm{feat}}$.

\textbf{Concatenation and Feature Fusion Head.} In tandem with these coupled branches, the multi-stage outputs are concatenated and processed by a compact fusion head to produce the final prediction:
\begin{equation}
\hat{y} = h\left([y_1, y_2, y_3]\right) + y_2 + y_3,
\label{eq:fusion}
\end{equation}
where $[\cdot]$ denotes concatenation and $h: \mathbb{R}^{D/2 + 2} \to \mathbb{R}$ is a shallow fusion MLP. This additive fusion couples the progressively narrowed signals from the final branches ($y_2$ and $y_3$) with a multi-scale summary of the entire cascade. By processing the concatenated vector $[y_1, y_2, y_3]$, the fusion head learns residual corrections that integrate coarse and fine-grained information, improving both predictive stability and expressivity at negligible computational cost.

\paragraph{Complexity Remarks.}
For fixed token dimension $d_{\mathrm{tok}}$, number of heads $H$, head dimension $d_h$, and sparse hyperparameters $(s,b, r, m)$, the computational complexity of TabNSM is dominated by three terms:
(i) \textit{Sparse Attention:} $\mathcal{O}(D d_{\mathrm{tok}} L)$, where $L$ is the number of attended keys per feature token, ensuring sub-quadratic scaling in $D$;
(ii) \textit{FTM:}  $O(Dd_{\mathrm{tok}} + DB_f)$ under fixed token dimension and feature mini-batch size $B_f$; and 
(iii) \textit{TabNSM Head:} $\mathcal{O}(D)$ through linear-time projections and shallow MLPs.
Empirically, this architecture achieves near-linear scaling with respect to the feature dimension $D$ under practical sparse mixer configurations, making it suitable for high-dimensional tabular datasets.

\subsection{GridLoss: Differentiable Soft-Binning for Ordinal Alignment}
Regression losses such as $\ell_1$ and Huber penalize pointwise errors but do not explicitly enforce \emph{ordinal alignment} between predictions and targets. To preserve relative positions along the target scale, we introduce GridLoss, a differentiable soft-binning objective that aligns predictions and targets in a shared discretized index space.

\textbf{Batch-Adaptive Grid Construction.} For a mini-batch of size $N$, let $\hat{y}, y \in \mathbb{R}^N$ denote the predictions and targets. We define a batch-adaptive interval $[\alpha,\beta]$ as:
\begin{equation}
\alpha = \min(\min(\hat{y}), \min(y)), \ \beta = \max(\max(\hat{y}), \max(y)).
\end{equation}
We construct a uniform grid of $K+1$ anchor points $e_k = \alpha + k \cdot \frac{\beta-\alpha}{K}$ for $k \in \{0, \dots, K\}$. A small constant $\epsilon > 0$ (e.g., $10^{-8}$) is added to $\beta-\alpha$ to prevent degenerate ranges. In implementation, $\alpha$ and $\beta$ are computed per mini-batch and detached from the computation graph.

\textbf{Soft Assignment via Temperature-Scaled Distances.} We compute the distances of each sample to all anchor points: $d^{\text{pred}}_{k} = |\hat{y} - e_k|$ and $d^{\text{true}}_{k} = |y - e_k|$. These are converted into soft assignments via a temperature-scaled softmax:
\[
w^{\text{pred}}_{k}= \frac{\exp(-d^{\text{pred}}_{k}/\tau)} {\sum_{j=0}^{K}\exp(-d^{\text{pred}}_{j}/\tau)}, \quad
w^{\text{true}}_{k}= \frac{\exp(-d^{\text{true}}_{k}/\tau)} {\sum_{j=0}^{K}\exp(-d^{\text{true}}_{j}/\tau)},
\] 
where $\tau > 0$ controls the assignment softness. We then compute soft indices as expectations over the anchors: $b^{\text{pred}} = \sum_{k=0}^{K} k w^{\text{pred}}_{k}$ and $b^{\text{true}} = \sum_{k=0}^{K} k w^{\text{true}}_{k}$.

\textbf{Objective Function.} The GridLoss objective is the mean $\ell_1$ distance between the soft indices:
\begin{equation}
\mathcal{L}_{\text{Grid}} = \frac{1}{N} \sum_{i=1}^{N} |b^{\text{pred}}_i - b^{\text{true}}_i|.
\end{equation}
GridLoss operates in an ordinal index space, encouraging predictions to fall into the correct relative region of the target distribution. As formally analyzed in Theorem~\ref{thm:gridloss_consistency} and Proposition~\ref{prop:gridloss_interpolation} (Appendix~\ref{app:gridloss_consistency}), under the high-resolution and low-temperature limit ($K \to \infty$, $\tau \to 0^+$), the normalized objective approaches an MAE-like ordinal distance. Thus, GridLoss provides a smooth, structure-aware approximation to robust pointwise regression losses while preserving ordinal information in the target space.

GridLoss has three useful properties: its soft index is monotone in the prediction, its gradient is bounded by $K/(2\tau)$, and it is Lipschitz continuous. 
For finite $\tau$, these properties imply stable gradients, while the high-resolution, low-temperature limit yields an MAE-like behavior; formal statements and proofs are provided in Appendix~\ref{sec_GridlossProperties}.

\subsection{Training Strategy with RISE Sampling} \label{sec_RISE_sampling}
Standard regression pipelines typically construct mini-batches via uniform sampling, implicitly treating all training instances as equally informative. In high-dimensional tasks with long-tailed error distributions, this can underrepresent ``hard'' samples that often correspond to minority subpopulations or extreme target tails, leading to suboptimal generalization. To mitigate this, we adopt a difficulty-aware sampling strategy termed RISE (Reweighted Instance Sampling by Error). RISE periodically estimates per-sample prediction errors and dynamically biases the sampling distribution toward harder examples. It is conceptually related to the non-uniform sampling used by GRIP~\citep{GRIP}, but RISE operates on regression examples and assigns bounded weights from empirical loss-quantile bins rather than sampling policy trajectories.

The RISE-enhanced training procedure adapts the model's focus as convergence progresses. Every $k$ epochs, we refresh the sampling distribution via the following steps:
\begin{enumerate}[label=\textbf{\arabic*.},leftmargin=*,topsep=2pt,itemsep=1pt,parsep=1.5pt]
\item \textbf{Loss Estimation:} Compute per-sample losses $\ell_i$ for the entire training set using the current parameters $\theta_e$.
\item \textbf{Quantile Partitioning:} Partition samples into $G$ difficulty bins based on the loss quantiles.
\item \textbf{Weight Assignment:} Assign a sampling weight $w_i = 1 + g_i \cdot \frac{\beta - 1}{G - 1}$, where $g_i \in \{0, \dots, G-1\}$ is the bin index and $\beta \ge 1$ is the maximum boost factor.
\item \textbf{Weighted Sampling:} Draw mini-batches using a weighted random sampler with probabilities $p_i \propto w_i$.
\end{enumerate}
By upweighting bins with larger errors, RISE implements a dynamic curriculum that prioritizes samples currently underfit by the model or lying in difficult regions of the target distribution.
We provide a theoretical interpretation of this mechanism below; see Appendix~\ref{sec_objective_interpretation} for the formal statement and proof.
\begin{proposition}[RISE as a Tail-aware Objective (Informal)]
\label{prop:rise_brief}
Let $\ell_i(\theta)$ denote the per-sample loss. At refresh step $e$, RISE assigns weights $w_i^{(e)} = \varphi(q_i^{(e)})$ based on the loss-quantile bin $q_i^{(e)}$. If mini-batches are drawn by weighted sampling with probabilities $p_i^{(e)} \propto w_i^{(e)}$ and updates are performed using sampled gradients without importance-weight correction, the expected SGD step is equivalent to descending the \emph{reweighted} empirical risk:
\begin{equation}
\mathcal{L}_{\mathrm{RISE}}^{(e)}(\theta) = \frac{\sum_{i=1}^N w_i^{(e)}\,\ell_i(\theta)}{\sum_{j=1}^N w_j^{(e)}}.
\label{eq:rise_objective}
\end{equation}
Thus, a monotonically increasing $\varphi(\cdot)$ ensures the optimization emphasizes the tail of the error distribution.
\end{proposition}
\noindent\textbf{Remark.} Standard Importance Sampling applies an inverse-probability correction ($1/p_i$) to keep the gradient estimator unbiased. RISE \emph{intentionally} omits this correction to prioritize high-loss samples, thus minimizing a risk functional that favors robustness in difficult feature-space regions.

\section{Experiments and Results} \label{results}

We evaluate TabNSM on nine real-world regression benchmarks spanning robotics, transportation, public health, criminology, real estate, air quality, and neuroimaging. Dataset statistics are reported in Table~\ref{table:datasets}, with detailed descriptions in Appendix~\ref{app:datasets}. We use a 72/8/20 train/validation/test split with fixed random partitions across methods. Samples with missing targets are removed, missing feature values are imputed with zero, and continuous features are standardized using training-set statistics. Models are trained for up to 100 epochs with early stopping on validation loss; test performance is reported from the best validation checkpoint. Additional seed-variation analysis is provided in Appendix~\ref{appendix:seed}.

\paragraph{Training protocol.}
The ASIM hyperparameters are selected via Optuna~\citep{optuna} on the validation set and fixed thereafter. The sparse operator $\mathcal{A}_{\mathrm{sparse}}$ is implemented using structured sparse attention (Eq.~\ref{eq:sparse-attn}), while $\mathcal{M}$ is instantiated as an FTM operator for global feature interaction. Unless otherwise stated, TabNSM is trained with AdamW, validation-based scheduling, mixed precision for first-order optimizers, and RISE-based sampling.

\paragraph{Main results.}
Table~\ref{tab:regression_results_compact} summarizes RMSE results against 38 baselines grouped into classical methods, tree ensembles, MLP-style models, transformer-based models, retrieval-based methods, other deep tabular architectures, and tabular foundation models (TFMs). Baseline implementation details are provided in Appendix~\ref{app:baselines}. To keep the main table compact, we report the best-performing baseline within each model family; the full per-model comparison is provided in Appendix~\ref{app:full_results}.

\begin{table}[!htbp]
\centering
\caption{RMSE on nine regression benchmarks. Lower is better. Best results are shown in \textbf{bold}, and second-best results are \underline{underlined}. We report the best-performing baseline within each model family; full per-model results are provided in Appendix~\ref{app:full_results}.}
\label{tab:regression_results_compact}
\begin{threeparttable}
\resizebox{\columnwidth}{!}{%
\begin{tabular}{lccccccccc}
\toprule
\textbf{Model} & SA & CP & AQ & TO & SC & CD & ChD & EVP & RES \\
\midrule
Best Classical  (6 models)  & 0.175 & 3.719 & 12.816 & \underline{0.027} & 3.602 & 0.236 & 8717 & 26.104 & 3163 \\
Best Tree-based (5 models)  & 0.171 & \textbf{2.617} & 5.869 & \underline{0.027} & 3.455 & \textbf{0.004} & 17525 & 2.799 & \underline{3135} \\
Best MLP-style  (5 models)  & 0.167 & 2.920 & 4.578 & 0.028 & 3.920 & 0.028 & 4963 & 1.556 & 3161 \\
Best Transformer (6 models) & 0.156 & 2.901 & \underline{4.081}& 0.028 & 3.667 & \underline{0.007} & \underline{4475} & \underline{1.507} & 3163 \\
Best Retrieval  (4 models)  & 0.150 & 2.833 & 5.062 & 0.028 & \underline{3.454} & 0.052 & 22661 & 2.653 & 3187 \\
Best Other      (8 models)  & \underline{0.149} & 2.991 & 4.690 & 0.028 & 3.570 & 0.136 & 4808 & 3.219 & 3163 \\
Best Foundation (4 models)  & 0.168 & \underline{2.740} & 6.748 & \underline{0.027} & 3.522 & 5.253 & 17858 & 65.471 & 4171 \\
\midrule
\textbf{TabNSM (Ours)} & \textbf{0.137} & 2.842 & \best{3.999} & \textbf{0.021} & \textbf{3.322} & \underline{0.007} & \textbf{4460} & \textbf{1.199} & \textbf{3103} \\
\bottomrule
\end{tabular}%
}
\begin{tablenotes}[flushleft]
\footnotesize
\item[]Per-family best is the minimum RMSE among runnable models in that family. Full per-model results are provided in Table~\ref{tab:regression_results} \\
in Appendix~\ref{app:full_results}.
\end{tablenotes}
\end{threeparttable}
\end{table}

TabNSM achieves the lowest RMSE on seven of nine benchmarks (SA, TO, AQ, SC, ChD, EVP, RES). 
Pairwise Wilcoxon signed-rank tests over the benchmark suite show statistically significant aggregate gains over the baselines after correction ($p \le 0.02$; Appendix~\ref{app:stats}). The largest margins appear on TO (+22\%, 0.021 vs.\ 0.027 for all competing models) and EVP (+21\%, 1.199 vs.\ 1.507 for FT-Transformer), two benchmarks with high feature dimensionality and structured or complex targets. On SA, TabNSM surpasses DNNR (0.149$\to$0.137), the strongest deep learning-based baseline. On AQ, TabNSM outperforms DCN-v2 (3.999 vs.\ 4.081) despite DCN-v2's explicit cross-network interaction design. 
On CD, TabNSM ranks second to Extratrees (0.007 vs. 0.004); on CP, a medium-sized, low-dimensional benchmark, tree-based ensembles dominate, making it the only dataset where TabNSM does not place in the top two.

\paragraph{Result interpretation.}
TabNSM is most effective on high-dimensional, interaction-rich datasets. On EVP (392 features) and ChD (405 features), ASIM's instance-adaptive foreground selection helps isolate informative subsets from a large noisy feature background, 
while FTM propagates global context without dense quadratic cost. On AQ, sparse selection also improves over DCN-v2's explicit cross-network mixing, suggesting that instance-specific interaction structure is important. The gains on TO are consistent with the role of GridLoss in preserving ordinal target alignment, while the strong performance on SA and RES suggests that RISE can help stabilize learning under heterogeneous target distributions.

\subsection{Ablation of Sparse Interaction Modeling}\label{sec_ablation_NSA}
We ablate ASIM by replacing its sparse feature-wise attention with either full self-attention or a Mamba block operating along the feature axis, while keeping the downstream head and training protocol fixed. As shown in Figure~\ref{fig:ablation}, full attention is unstable in high-dimensional settings, while Mamba improves efficiency but still underperforms ASIM. This indicates that linear-time mixing alone is insufficient: instance-adaptive sparse feature selection is important for robust high-dimensional regression. These findings are consistent with Proposition~\ref{prop:noise_short}, which shows that sparse selection can improve SNR when the selected support covers informative features and excludes many background dimensions. Per-instance feature attribution heatmaps further confirm that TabNSM activates distinct feature subsets across samples (Appendix~\ref{app:interpretability}).
\begin{figure}[t]
    \centering
    \begin{subfigure}[t]{0.49\linewidth}
        \centering
        \includegraphics[height=3.6cm]{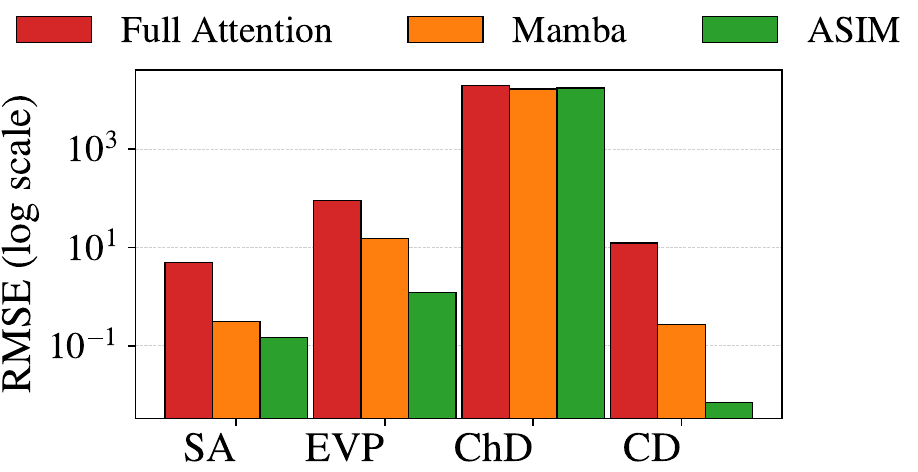}
        \caption{Ablation of sparse interaction modeling.}
        \label{fig:ablation}
    \end{subfigure}
    \hfill
    \begin{subfigure}[t]{0.49\linewidth}
        \centering
        \includegraphics[height=3.6cm, width=5.7cm]{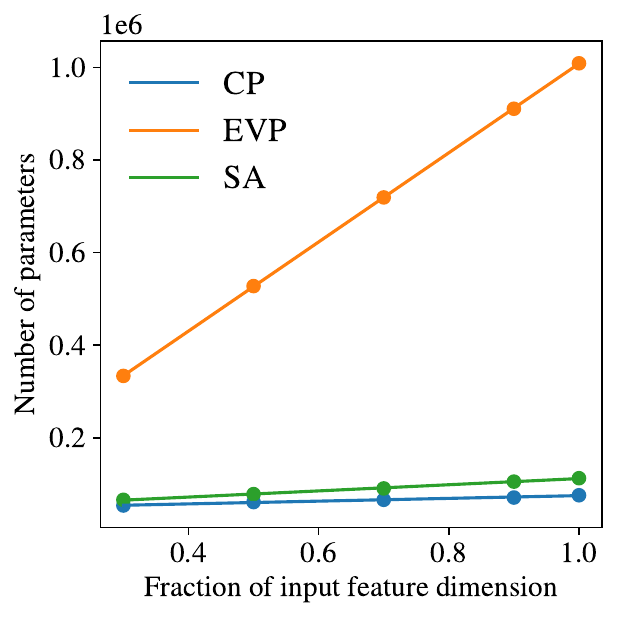}
        \caption{Parameter scaling with feature dimensionality.}
        \label{fig:param_D}
    \end{subfigure}
    \caption{\textbf{Ablation and scalability analyses.}
(a) Full attention, Mamba, and ASIM within TabNSM; ASIM achieves the lowest RMSE. 
(b) Parameter count grows approximately linearly with retained feature fraction, consistent with the ASIM complexity analysis.}
    \label{fig:ablation_scaling}
    \vspace{-10pt}
\end{figure}

\subsection{Scalability Analysis}
\label{sec:scalability}

To complement the empirical accuracy comparison, we examine whether the ASIM component exhibits the near-linear scaling predicted by Proposition~\ref{prop:ASIM_complexity_brief}. We vary the retained feature fraction while keeping the sparse-attention configuration fixed and measure the resulting parameter count. As shown in Figure~\ref{fig:param_D}, the number of parameters grows approximately linearly with the input feature dimension across datasets. This behavior supports the intended design of ASIM: sparse feature-wise interaction modeling avoids dense all-pair feature mixing while retaining sufficient capacity for high-dimensional tabular regression. Additional runtime and memory measurements are provided in Appendix~\ref{Memory}.

\paragraph{Additional component checks.}
Appendix~\ref{appendix:additional_ablation} further evaluates FTM, GridLoss, and RISE. Removing FTM degrades performance on CP, TO, SA, and EVP, especially on EVP, supporting its role in propagating information beyond the sparse attention support. GridLoss matches or improves over standard pointwise losses under heterogeneous errors, while RISE reduces validation variance by prioritizing difficult samples. Together, these results show that TabNSM's gains arise from sparse feature selection, global mixing, structure-aware supervision, and difficulty-aware sampling.

\section{Conclusion and Discussion}
We introduced TabNSM, a regression-focused extension of our earlier TabNSA and TabMixer architectures that combines an adapted sparse interaction module with a Multi-Stage Regression Head, ordinal grid supervision (GridLoss), and difficulty-aware reweighting (RISE). ASIM replaces dense quadratic interaction with sample-specific foreground selection, achieving near-linear complexity while preserving adaptive sparsity. Across nine regression benchmarks, TabNSM achieves the best performance on seven datasets and ranks second on one, outperforming GBDT baselines on most tasks. Ablations show that instance-adaptive selectivity, rather than efficiency alone, is a major performance driver, while the regression-specific head, GridLoss, and RISE provide complementary gains. Although TabNSM requires GPU training and can incur higher overhead than tree-based methods, it achieves substantially lower error than most deep learning baselines with favorable training efficiency. Future work will pursue compression and distillation, extend GridLoss to multi-target settings, and explore foreground--background interactions for multimodal tabular learning.
\newpage

\section{Acknowledgements} We acknowledge the contributors who made the publicly available datasets and baseline models accessible for research purposes. We also extend our sincere thanks to Brian Gold, PhD, and the Gold Lab at the University of Kentucky for their continued support and for providing the facilities required to conduct this study. This work was partially supported by the National Science Foundation (NSF) under Grants IIS-2327113 and ITE-2433190, the National Institutes of Health (NIH) under Grants R21AG070909 and P30AG072946, and by the National Artificial Intelligence Research Resource (NAIRR) Pilot through NSF OAC-240219, as well as the Jetstream2, Bridges2, and Neocortex computing resources.
Data for the SC benchmark were obtained through NACC/SCAN under the NACC Data.
\bibliography{references}
\bibliographystyle{unsrtnat}

\appendix
\onecolumn

\section{Theoretical Intuition of GridLoss}

\paragraph{Gradients and Monotonicity.}
GridLoss maps a scalar value $x$ to a soft grid index
\[
b(x)=\sum_{k=0}^{K} k\,w_k(x),
\qquad
w_k(x)=\frac{\exp(-|x-e_k|/\tau)}
{\sum_{j=0}^{K}\exp(-|x-e_j|/\tau)} .
\]
This mapping can be viewed as a smooth approximation to hard binning: for finite temperature $\tau$, the weights concentrate around nearby grid anchors, so $b(x)$ behaves like a differentiable coordinate on the target grid. The derivative $\partial b(x)/\partial x$ is controlled by the temperature and grid resolution, yielding bounded gradients for finite $\tau$. As $x$ increases, mass shifts from lower-index anchors to higher-index anchors, making $b(x)$ nondecreasing on an ordered grid. Thus, GridLoss provides an ordinal training signal: it encourages predictions not only to reduce pointwise error, but also to occupy the correct relative position along the target scale.

\paragraph{Why Not KL/Wasserstein Distribution Matching.}
Distributional objectives such as KL divergence or Wasserstein distance typically require constructing normalized distributions, such as histograms or kernel-smoothed densities, over predicted and true targets. In per-sample regression, this introduces additional design choices, including binning, smoothing, regularization, and, for optimal transport, potentially iterative solvers. Moreover, batch-level distribution matching can encourage agreement between marginal target distributions without directly enforcing alignment between each prediction--target pair. GridLoss instead performs a sample-wise comparison by mapping each $(\hat{y}_i,y_i)$ pair onto a shared grid and penalizing their soft-index distance. This captures ordinal geometry with low overhead, controlled mainly by the grid resolution $K$ and temperature $\tau$.

\paragraph{Why Pointwise Regression Losses Alone Can Be Insufficient.}
Standard pointwise losses optimize residual magnitudes but do not explicitly encode ordinal placement on the target scale. MSE and RMSE emphasize large residuals and can over-focus on outliers; MAE and robust variants, such as Huber and log-cosh, improve robustness but still compare predictions and targets only through pointwise deviations. MAPE can be unstable near zero targets and may overweight small values, while quantile loss targets conditional quantiles rather than the conditional mean. GridLoss complements these objectives by providing an ordinal, calibration-oriented signal in grid space.

\paragraph{Ablation Protocol.}
We ablate two aspects of GridLoss. First, we compare different regression objectives, including MSE, MAE, Huber, log-cosh, GridLoss, and mixed objectives
\[
\lambda \mathcal{L}_{\mathrm{Huber}} + (1-\lambda)\mathcal{L}_{\mathrm{Grid}},
\qquad
\lambda \in \{1.0,0.9,0.7,0.5,0.3,0.0\}.
\]
Second, we study the sensitivity of GridLoss to its hyperparameters, using
\[
K\in\{50,100,200,500,1000\},
\qquad
\tau\in\{0.1,0.3,1.0,3.0,10.0\}.
\]
In addition to RMSE, MAE, and $R^2$, we report ordinal and calibration-oriented metrics, including Spearman's $\rho$ and bin accuracy within $\pm 1$ bin, to quantify improvements in target-scale alignment.

\section{Proofs of Theoretical Results}

\subsection{ASIM Complexity}
\label{sec_proof_ASIM}

The full statement of Proposition~\ref{prop:ASIM_complexity_brief} is as follows.

\begin{proposition}[Time and Memory Complexity of ASIM]
\label{prop:ASIM_complexity_full}
Let $D$ be the number of scalar features and let $d_{\mathrm{tok}}$ be the token dimension per feature. Let
$X_{\mathrm{tok}}\in\mathbb{R}^{D\times d_{\mathrm{tok}}}$ denote the tokenized feature representation. ASIM computes
\begin{align}
\tilde{X} 
&= \mathcal{A}_{\mathrm{sparse}}(X_{\mathrm{tok}};s,b,r,m,H,d_h)
\in \mathbb{R}^{D\times d_{\mathrm{tok}}}, \\
U 
&= w_1X_{\mathrm{tok}}+\tilde{X}
\in \mathbb{R}^{D\times d_{\mathrm{tok}}}, \\
Z 
&= \mathcal{M}(U)
\in \mathbb{R}^{D\times d_{\mathrm{tok}}}, \\
Y 
&= w_2U+Z+w_3X_{\mathrm{tok}}
\in \mathbb{R}^{D\times d_{\mathrm{tok}}},
\end{align}
followed by $y_{\mathrm{feat}}=\operatorname{mean}_{\mathrm{tok}}(Y)\in\mathbb{R}^{D}$.
Assume the sparse attention operator attends, for each query feature, to at most
\[
L := L(s,b,r,m)
\]
feature positions, where $L$ is determined by the fixed sparse configuration. If $d_{\mathrm{tok}}=H d_h$, then the sparse interaction part of one ASIM layer has complexity
\[
\mathcal{O}(D H d_h L)=\mathcal{O}(D d_{\mathrm{tok}}L)
\]
and activation memory
\[
\mathcal{O}\!\left(D(d_{\mathrm{tok}}+L)\right).
\]
Including the token projections, FTM operator, residual blending, and token-wise aggregation, the overall ASIM computation remains linear in $D$ under fixed $d_{\mathrm{tok}}$, $H$, $d_h$, fixed feature mini-batch size $B_f$, and sparse hyperparameters $(s,b,r,m)$.
\end{proposition}

\begin{proof}
The sparse attention path restricts each query feature to at most $L$ attended feature positions. Across $D$ query features and $H$ attention heads with head dimension $d_h$, the sparse dot-product and value-aggregation cost is
\[
\mathcal{O}(D H d_h L)=\mathcal{O}(D d_{\mathrm{tok}}L),
\]
using $d_{\mathrm{tok}}=H d_h$. If sparse attention weights or indices are materialized, the sparse attention pattern contributes $\mathcal{O}(DL)$ storage, while the token representations contribute $\mathcal{O}(D d_{\mathrm{tok}})$ storage. Thus, the sparse attention activation memory is $\mathcal{O}(D(d_{\mathrm{tok}}+L))$.

The token projections contribute $\mathcal{O}(D d_{\mathrm{tok}}^2)$ computation under standard linear projections, which is linear in $D$ for fixed $d_{\mathrm{tok}}$. 
The FTM operator $\mathcal{M}(\cdot)$ applies lightweight embedding-wise mixing and feature-wise mini-batch mixing to $U\in\mathbb{R}^{D\times d_{\mathrm{tok}}}$, contributing $O(Dd_{\mathrm{tok}}+DB_f)$ computation under fixed token dimension and fixed feature mini-batch size $B_f$. The residual blending operations and the token-wise mean projection also require $O(Dd_{\mathrm{tok}})$ computation and storage. Therefore, under fixed sparse hyperparameters, fixed token dimension, and fixed feature mini-batch size, both compute and activation memory scale linearly with the feature dimension $D$.
\end{proof}

\subsection{Conditional SNR Improvement under Sparse Feature Selection}
\label{app:noise_proof}

We provide the formal statement and proof of Proposition~\ref{prop:noise_short}. The result is a simplified conditional analysis: if the sparse support covers the informative features while excluding most background dimensions, then sparse aggregation improves the signal-to-noise ratio relative to dense aggregation.

\begin{proposition}[Conditional SNR Improvement under Sparse Feature Selection]
\label{prop:noise_full}
Let $\mathbf{x}=[\mathbf{x}_{\mathcal{S}};\mathbf{x}_{\mathcal{N}}]\in\mathbb{R}^D$ be partitioned into informative features indexed by $\mathcal{S}$, with $|\mathcal{S}|=s$, and background noise features indexed by $\mathcal{N}=\{1,\ldots,D\}\setminus\mathcal{S}$. Assume that the noise features are independent with mean zero and variance $\sigma^2$. Let
\[
\mu_{\mathcal{S}}=\frac{1}{s}\sum_{d\in\mathcal{S}}x_d
\]
denote the average signal over informative features. Define
\[
A_{\mathrm{dense}}(\mathbf{x})
=
\frac{1}{D}\sum_{d=1}^{D}x_d,
\qquad
A_{\mathrm{sparse}}(\mathbf{x};\mathcal{T})
=
\frac{1}{|\mathcal{T}|}\sum_{d\in\mathcal{T}}x_d,
\]
where $\mathcal{T}\subseteq\{1,\ldots,D\}$ is a sparse support satisfying $\mathcal{S}\subseteq\mathcal{T}$ and $|\mathcal{T}|=mb$.

The noise variance contributions are
\[
\operatorname{Var}_{\mathcal{N}}[A_{\mathrm{dense}}]
=
\frac{D-s}{D^2}\sigma^2,
\qquad
\operatorname{Var}_{\mathcal{N}}[A_{\mathrm{sparse}}]
=
\frac{mb-s}{(mb)^2}\sigma^2.
\]
Moreover, the corresponding signal-to-noise ratios are
\[
\operatorname{SNR}_{\mathrm{dense}}
=
\frac{s|\mu_{\mathcal{S}}|}{\sigma\sqrt{D-s}},
\]
and, for $mb>s$,
\[
\operatorname{SNR}_{\mathrm{sparse}}
=
\frac{s|\mu_{\mathcal{S}}|}{\sigma\sqrt{mb-s}}.
\]
Thus, when $\mathcal{T}$ covers the informative features while excluding most background dimensions, sparse aggregation improves SNR by a factor of
\[
\frac{\operatorname{SNR}_{\mathrm{sparse}}}
{\operatorname{SNR}_{\mathrm{dense}}}
=
\sqrt{\frac{D-s}{mb-s}},
\qquad mb>s.
\]
In the ideal case $mb=s$, the sparse support contains no background noise, so the noise contribution vanishes under this simplified model.
\end{proposition}

\begin{proof}
For dense aggregation,
\[
A_{\mathrm{dense}}(\mathbf{x})
=
\frac{1}{D}\sum_{d\in\mathcal{S}}x_d
+
\frac{1}{D}\sum_{d\in\mathcal{N}}x_d.
\]
Treating the informative features as fixed, the variance comes only from the noise term:
\[
\operatorname{Var}_{\mathcal{N}}[A_{\mathrm{dense}}]
=
\operatorname{Var}\!\left[
\frac{1}{D}\sum_{d\in\mathcal{N}}x_d
\right]
=
\frac{1}{D^2}\sum_{d\in\mathcal{N}}\operatorname{Var}(x_d)
=
\frac{D-s}{D^2}\sigma^2.
\]
The signal component of dense aggregation is
\[
\mathbb{E}_{\mathcal{N}}[A_{\mathrm{dense}}]
=
\frac{1}{D}\sum_{d\in\mathcal{S}}x_d
=
\frac{s}{D}\mu_{\mathcal{S}}.
\]
Hence
\[
\operatorname{SNR}_{\mathrm{dense}}
=
\frac{|s\mu_{\mathcal{S}}/D|}
{\sigma\sqrt{D-s}/D}
=
\frac{s|\mu_{\mathcal{S}}|}
{\sigma\sqrt{D-s}}.
\]

For sparse aggregation, since $\mathcal{S}\subseteq\mathcal{T}$ and $|\mathcal{T}|=mb$, the number of background noise features included in $\mathcal{T}$ is $mb-s$. Therefore,
\[
\operatorname{Var}_{\mathcal{N}}[A_{\mathrm{sparse}}]
=
\operatorname{Var}\!\left[
\frac{1}{mb}\sum_{d\in\mathcal{T}\cap\mathcal{N}}x_d
\right]
=
\frac{mb-s}{(mb)^2}\sigma^2.
\]
The signal component is
\[
\mathbb{E}_{\mathcal{N}}[A_{\mathrm{sparse}}]
=
\frac{1}{mb}\sum_{d\in\mathcal{S}}x_d
=
\frac{s}{mb}\mu_{\mathcal{S}}.
\]
For $mb>s$, this yields
\[
\operatorname{SNR}_{\mathrm{sparse}}
=
\frac{|s\mu_{\mathcal{S}}/mb|}
{\sigma\sqrt{mb-s}/mb}
=
\frac{s|\mu_{\mathcal{S}}|}
{\sigma\sqrt{mb-s}}.
\]
Taking the ratio gives
\[
\frac{\operatorname{SNR}_{\mathrm{sparse}}}
{\operatorname{SNR}_{\mathrm{dense}}}
=
\sqrt{\frac{D-s}{mb-s}}.
\]
If $mb=s$, then $\mathcal{T}$ contains no background noise features, so the noise variance of $A_{\mathrm{sparse}}$ is zero. This completes the proof.
\end{proof}

\begin{corollary}[Scaling with Feature Dimension]
\label{cor:scaling}
Fix the number of informative features $s$ and let the sparse budget be $mb=s+\delta$ for a fixed $\delta\ge 0$. As $D\to\infty$, the dense aggregation SNR satisfies
\[
\operatorname{SNR}_{\mathrm{dense}}
=
\frac{s|\mu_{\mathcal{S}}|}{\sigma\sqrt{D-s}}
=
\mathcal{O}(D^{-1/2}),
\]
whereas, for $\delta>0$,
\[
\operatorname{SNR}_{\mathrm{sparse}}
=
\frac{s|\mu_{\mathcal{S}}|}{\sigma\sqrt{\delta}}
=
\Theta(1).
\]
If $\delta=0$, the sparse support contains no background noise and the noise contribution vanishes under the simplified model. Thus, dense aggregation suffers from signal dilution as the ambient feature dimension grows, whereas sparse aggregation preserves a dimension-independent SNR when the sparse support remains tight.
\end{corollary}

\begin{proof}
The dense SNR expression follows directly from Proposition~\ref{prop:noise_full}:
\[
\operatorname{SNR}_{\mathrm{dense}}
=
\frac{s|\mu_{\mathcal{S}}|}{\sigma\sqrt{D-s}}
=
\mathcal{O}(D^{-1/2})
\]
for fixed $s$, $\mu_{\mathcal{S}}$, and $\sigma$. For sparse aggregation with $mb=s+\delta$ and $\delta>0$,
\[
\operatorname{SNR}_{\mathrm{sparse}}
=
\frac{s|\mu_{\mathcal{S}}|}{\sigma\sqrt{mb-s}}
=
\frac{s|\mu_{\mathcal{S}}|}{\sigma\sqrt{\delta}},
\]
which is independent of $D$. When $\delta=0$, the sparse aggregation contains no background noise, so its noise variance is zero under the assumptions of the proposition.
\end{proof}

\begin{remark}[Extension to Softmax Attention Weights]
\label{rem:softmax_extension}
The uniform-weight analysis provides a tractable baseline. For softmax attention with learned queries and keys, the SNR behavior depends on the model's ability to assign small weights to background features and larger weights to informative features. If the attention mechanism concentrates most mass on $\mathcal{S}$, then the effective support
\[
\mathcal{T}_{\mathrm{eff}}=\{d:\alpha_d>\varepsilon\}
\]
can approach the idealized sparse case. This motivates structurally sparse mechanisms, which encourage restricted feature interaction patterns rather than relying solely on learned dense softmax weights to become sparse in high-dimensional noisy settings.
\end{remark}

\begin{remark}[Connection to Empirical Results]
\label{rem:empirical_connection}
The conditional analysis in Proposition~\ref{prop:noise_full} and Corollary~\ref{cor:scaling} provides intuition for the empirical findings in Section~\ref{sec_ablation_NSA} and Figure~\ref{fig:ablation}:
\begin{enumerate}
    \item \textbf{Dense attention.} In high-dimensional settings, dense attention mixes informative and background features through all-pair interactions. The analysis suggests that, when only a small subset of features is informative, dense aggregation can dilute the useful signal as $D$ grows.
    
    \item \textbf{Mamba.} SSM-based mixing provides linear-time propagation along the feature axis, but it does not explicitly enforce sparse feature selection. As a result, it may propagate both informative and background dimensions, yielding intermediate performance.
    
    \item \textbf{ASIM.} Structured sparse selection encourages the model to restrict feature interactions to a smaller support. When this support covers informative features while excluding many irrelevant dimensions, the analysis predicts improved SNR. This provides theoretical intuition for why the ASIM design can improve robustness in high-dimensional tabular regression.
\end{enumerate}
\end{remark}

\subsection{Properties of GridLoss}
\label{sec_GridlossProperties}

\begin{lemma}[Monotonicity of the GridLoss soft index]
\label{lem:gridloss_monotone}
Fix ordered grid anchors $e_0 < e_1 < \cdots < e_K$ and temperature $\tau>0$.
For any scalar $x\in\mathbb{R}$, define
\[
w_k(x) \;=\; \frac{\exp\!\left(-|x-e_k|/\tau\right)}{\sum_{j=0}^K \exp\!\left(-|x-e_j|/\tau\right)},
\qquad
b(x) \;=\; \sum_{k=0}^K k\, w_k(x).
\]
Then $b(x)$ is nondecreasing in $x$; equivalently, $x_1\le x_2$ implies $b(x_1)\le b(x_2)$.
\end{lemma}

\begin{proof}
Let $z_k(x)=-|x-e_k|/\tau$ and $w_k(x)=\operatorname{softmax}(z(x))_k$. For $x\notin\{e_0,\ldots,e_K\}$, we have
\[
z'_k(x)=-\frac{1}{\tau}\operatorname{sign}(x-e_k).
\]
The softmax derivative gives
\[
w'_k(x)
=
w_k(x)
\left(
z'_k(x)-\sum_{j=0}^K w_j(x)z'_j(x)
\right).
\]
Therefore,
\[
b'(x)
=
\sum_{k=0}^K k\,w'_k(x)
=
-\frac{1}{\tau}
\left(
\mathbb{E}_w[k s_k]
-
\mathbb{E}_w[k]\mathbb{E}_w[s_k]
\right)
=
-\frac{1}{\tau}\operatorname{Cov}_w(k,s_k),
\]
where $s_k=\operatorname{sign}(x-e_k)$ and $\mathbb{E}_w[\cdot]$ denotes expectation under the weights $\{w_k(x)\}_{k=0}^K$.
Since $e_k$ increases with $k$, the sequence $\{s_k\}_{k=0}^K$ is nonincreasing in $k$. Hence $\operatorname{Cov}_w(k,s_k)\le 0$, which implies $b'(x)\ge 0$ wherever the derivative exists. At anchor locations $x=e_k$, monotonicity follows by continuity and one-sided derivatives.
\end{proof}

\begin{lemma}[Bounded slope and stable GridLoss gradients]
\label{lem:gridloss_grad_bound}
Under the same definitions as Lemma~\ref{lem:gridloss_monotone}, for all $x\notin\{e_0,\ldots,e_K\}$,
\[
0 \le b'(x) \le \frac{K}{2\tau}.
\]
Consequently, for the per-sample GridLoss term
\[
\ell_{\mathrm{Grid}}(\hat{y},y)
=
\big|b(\hat{y})-b(y)\big|,
\]
any subgradient $g\in\partial_{\hat{y}}\ell_{\mathrm{Grid}}(\hat{y},y)$ satisfies
\[
|g|\le \frac{K}{2\tau}.
\]
\end{lemma}

\begin{proof}
From the proof of Lemma~\ref{lem:gridloss_monotone},
\[
b'(x)
=
-\frac{1}{\tau}\operatorname{Cov}_w(k,s_k),
\qquad s_k\in\{-1,+1\}.
\]
By Cauchy--Schwarz,
\[
|\operatorname{Cov}_w(k,s_k)|
\le
\sqrt{\operatorname{Var}_w(k)\operatorname{Var}_w(s_k)}
\le
\sqrt{\operatorname{Var}_w(k)},
\]
since $\operatorname{Var}_w(s_k)\le 1$. Because $k\in\{0,1,\ldots,K\}$, Popoviciu's inequality gives
\[
\operatorname{Var}_w(k)\le \frac{K^2}{4}.
\]
Therefore,
\[
0\le b'(x)\le \frac{K}{2\tau}.
\]
For $\ell_{\mathrm{Grid}}(\hat{y},y)=|b(\hat{y})-b(y)|$, the chain rule for subgradients yields
\[
\partial_{\hat{y}}\ell_{\mathrm{Grid}}(\hat{y},y)
\subseteq
\operatorname{sign}\!\big(b(\hat{y})-b(y)\big)
\cdot
\partial_{\hat{y}}b(\hat{y}),
\]
so every subgradient satisfies $|g|\le K/(2\tau)$.
\end{proof}

\begin{lemma}[Lipschitz continuity of the soft index]
\label{lem:gridloss_lipschitz}
Under the conditions of Lemma~\ref{lem:gridloss_grad_bound}, for all $x_1,x_2\in\mathbb{R}$,
\[
|b(x_1)-b(x_2)|
\le
\frac{K}{2\tau}|x_1-x_2|.
\]
\end{lemma}

\begin{proof}
On intervals not containing grid anchors, the result follows from the mean value theorem and the slope bound in Lemma~\ref{lem:gridloss_grad_bound}. Since $b$ is continuous at the anchors, the bound extends to arbitrary $x_1,x_2\in\mathbb{R}$ by partitioning the interval between $x_1$ and $x_2$ at any anchors it contains and summing the bounds over the subintervals.
\end{proof}

\subsection{Asymptotic Consistency of GridLoss}
\label{app:gridloss_consistency}

We show that normalized GridLoss recovers the mean absolute error in a high-resolution, low-temperature limit. This provides a link between GridLoss and classical robust regression losses.

\paragraph{Fixed-interval analysis versus batch-adaptive implementation.}
For theoretical clarity, the analysis assumes a fixed interval $[\alpha,\beta]$. In implementation, the grid interval is batch-adaptive: $\alpha$ and $\beta$ are computed from the current mini-batch, for example using the batch extrema of $y$ and $\hat{y}$, and are detached from the computation graph. The analysis applies conditionally within each batch-defined interval because the key arguments depend on the local grid geometry and the relative distances to nearby anchors.

\begin{theorem}[GridLoss Consistency]
\label{thm:gridloss_consistency}
Let $\mathcal{L}_{\mathrm{Grid}}^{(K,\tau)}(\hat{y},y)$ denote GridLoss with $K$ bins and temperature $\tau$ on a fixed interval $[\alpha,\beta]$. Define the normalized GridLoss
\[
\widetilde{\mathcal{L}}_{\mathrm{Grid}}^{(K,\tau)}(\hat{y},y)
=
\frac{\beta-\alpha}{K}
\mathcal{L}_{\mathrm{Grid}}^{(K,\tau)}(\hat{y},y).
\]
Then, for any $\hat{y},y\in[\alpha,\beta]$,
\[
\lim_{\tau\to 0^+}\lim_{K\to\infty}
\widetilde{\mathcal{L}}_{\mathrm{Grid}}^{(K,\tau)}(\hat{y},y)
=
|\hat{y}-y|.
\]
That is, normalized GridLoss converges to the $\ell_1$ loss as the grid becomes infinitely fine and the soft assignments become sharp.
\end{theorem}

\begin{proof}
Let
\[
\Delta=\frac{\beta-\alpha}{K},
\qquad
e_k=\alpha+k\Delta,
\qquad k=0,\ldots,K.
\]
For a scalar $x\in[\alpha,\beta]$, define the soft index
\[
b^{(K,\tau)}(x)=\sum_{k=0}^K k\,w_k(x),
\qquad
w_k(x)=
\frac{\exp(-|x-e_k|/\tau)}
{\sum_{j=0}^K\exp(-|x-e_j|/\tau)}.
\]
The normalized soft coordinate is $\Delta b^{(K,\tau)}(x)$. As $K\to\infty$ with $\tau$ fixed, the discrete grid becomes dense and the weighted sum converges to the expectation of a truncated Laplace-type density on $[\alpha,\beta]$ centered at $x$:
\[
\Delta b^{(K,\tau)}(x)
\longrightarrow
B_{\tau}(x)
:=
\frac{\int_{\alpha}^{\beta} u \exp(-|u-x|/\tau)\,du}
{\int_{\alpha}^{\beta} \exp(-|u-x|/\tau)\,du}.
\]
The quantity $B_{\tau}(x)$ is a smoothed version of $x$. Since the exponential kernel concentrates around $x$ as $\tau\to 0^+$, we have
\[
B_{\tau}(x)\to x
\qquad
\text{for every } x\in[\alpha,\beta].
\]
Therefore,
\[
\lim_{\tau\to 0^+}\lim_{K\to\infty}
\Delta b^{(K,\tau)}(x)
=
x.
\]
Applying this to $\hat{y}$ and $y$ gives
\[
\lim_{\tau\to 0^+}\lim_{K\to\infty}
\widetilde{\mathcal{L}}_{\mathrm{Grid}}^{(K,\tau)}(\hat{y},y)
=
\lim_{\tau\to 0^+}
\left|
B_{\tau}(\hat{y})-B_{\tau}(y)
\right|
=
|\hat{y}-y|,
\]
which completes the proof.
\end{proof}

\begin{corollary}[Approximation Behavior]
\label{cor:gridloss_approx}
On a fixed interval $[\alpha,\beta]$, normalized GridLoss approaches an MAE-like ordinal distance as the grid resolution increases and the temperature decreases. Finer grids reduce discretization error, while finite temperature smooths the hard binning operation and yields bounded gradients. Thus, $K$ controls target-scale resolution, whereas $\tau$ controls the smoothness--sharpness trade-off.
\end{corollary}

\begin{proof}
The convergence to MAE follows from Theorem~\ref{thm:gridloss_consistency}. The role of finite temperature follows from Lemma~\ref{lem:gridloss_grad_bound}, which bounds the gradient magnitude by $K/(2\tau)$ for fixed $K$ and $\tau$. Thus, larger $K$ gives a finer discretization of the target interval, while $\tau$ controls how sharply each scalar value is assigned to nearby grid anchors.
\end{proof}

\begin{remark}[Relationship to Soft Label Regression]
\label{rem:soft_label}
Theorem~\ref{thm:gridloss_consistency} positions GridLoss within a broader family of soft-binning regression losses. Related approaches include:
\begin{itemize}
    \item \textbf{Label Distribution Learning (LDL):} LDL represents targets as distributions over discrete labels; GridLoss can be viewed as using Gibbs-type soft assignments centered at $\hat{y}$ and $y$.
    
    \item \textbf{Ordinal Regression:} Ordinal regression methods often enforce rank consistency through hard or probabilistic class decompositions. GridLoss instead uses soft indices to provide a differentiable relaxation of ordinal positioning for continuous targets.
    
    \item \textbf{Histogram or Distributional Losses:} Histogram losses and Earth Mover's Distance compare target distributions. GridLoss compares each prediction--target pair through its soft grid index, yielding a sample-wise ordinal distance.
\end{itemize}
The consistency result shows that GridLoss approaches MAE in the high-resolution, low-temperature limit while maintaining smooth gradients for finite $\tau$.
\end{remark}

\begin{remark}[Practical Implications]
\label{rem:practical_gridloss}
Theorem~\ref{thm:gridloss_consistency} has several practical implications:
\begin{enumerate}
    \item \textbf{Hyperparameter guidance:} Increasing $K$ improves target-scale resolution, but very large $K$ with very small $\tau$ can increase gradient magnitudes, as Lemma~\ref{lem:gridloss_grad_bound} shows. Thus, $K$ and $\tau$ should be selected jointly to balance resolution and stability.
    
    \item \textbf{MAE-like behavior:} Since normalized GridLoss approaches MAE in the high-resolution, low-temperature limit, it inherits an MAE-like robustness profile while retaining smooth gradients at finite temperature.
    
    \item \textbf{Ordinal structure:} Unlike standard pointwise losses, GridLoss explicitly operates in a discretized ordinal space, encouraging predictions to respect relative placement along the target range.
\end{enumerate}
\end{remark}

\begin{proposition}[GridLoss as Interpolation]
\label{prop:gridloss_interpolation}
For fixed $K$ and varying $\tau$, GridLoss interpolates between:
\begin{enumerate}
    \item As $\tau\to\infty$, $w_k(x)\to 1/(K+1)$ for all $k$, so $b(x)\to K/2$ and $\mathcal{L}_{\mathrm{Grid}}^{(K,\tau)}(\hat{y},y)\to 0$ for all $\hat{y},y$.
    
    \item As $\tau\to 0^+$, the weights concentrate on the nearest anchor, so $\mathcal{L}_{\mathrm{Grid}}^{(K,\tau)}(\hat{y},y)$ converges to a hard ordinal bin distance.
\end{enumerate}
As $K$ increases, the grid provides a finer discretization of the target interval. Under the high-resolution and low-temperature limit in Theorem~\ref{thm:gridloss_consistency}, the normalized GridLoss converges to MAE.
\end{proposition}

\begin{proof}
As $\tau\to\infty$, each exponential term satisfies $\exp(-|x-e_k|/\tau)\to 1$, so $w_k(x)\to 1/(K+1)$ for all $k$. Hence
\[
b(x)\to \frac{1}{K+1}\sum_{k=0}^K k = \frac{K}{2}
\]
for all $x$, and the index distance between any $\hat{y}$ and $y$ converges to zero.

As $\tau\to 0^+$, the weights concentrate on the nearest grid anchor. Thus $b(x)$ converges to the nearest-anchor index, and GridLoss becomes the absolute distance between hard ordinal bin indices. Finally, the convergence of normalized GridLoss to MAE under joint high-resolution and low-temperature behavior follows from Theorem~\ref{thm:gridloss_consistency}.
\end{proof}

\subsection{Objective Interpretation}
\label{sec_objective_interpretation}

\begin{proposition}[Objective Interpretation of RISE as Reweighted Empirical Risk]
\label{prop:rise_objective_interpretation}
Let $\{(x_i,y_i)\}_{i=1}^N$ be the training set and let $\ell_i(\theta)$ denote the per-sample loss at parameters $\theta$, e.g.,
$\ell_i(\theta)=\ell(f_\theta(x_i),y_i)$. At RISE refresh step $e$, define a positive weight
\[
w_i^{(e)}=\varphi(q_i^{(e)})>0,
\]
where $q_i^{(e)}\in\{0,\ldots,G-1\}$ is the empirical loss-quantile bin index of sample $i$, computed from $\{\ell_j(\theta_e)\}_{j=1}^N$, and $\varphi(\cdot)$ is a monotone increasing bin-to-weight map. Consider SGD where each mini-batch is drawn by weighted sampling using probabilities
\[
p_i^{(e)}
=
\frac{w_i^{(e)}}{\sum_{j=1}^N w_j^{(e)}},
\]
and the update uses the uncorrected stochastic gradient $\nabla_\theta \ell_i(\theta)$ of the sampled example, without inverse-probability correction. Conditional on the weights $\{w_i^{(e)}\}$, the expected stochastic gradient equals the gradient of the reweighted empirical risk
\[
\mathcal{L}_{\mathrm{RISE}}^{(e)}(\theta)
=
\sum_{i=1}^N p_i^{(e)}\ell_i(\theta)
=
\frac{\sum_{i=1}^N w_i^{(e)}\ell_i(\theta)}
{\sum_{j=1}^N w_j^{(e)}}.
\]
That is,
\[
\mathbb{E}_{i\sim p^{(e)}}[\nabla_\theta \ell_i(\theta)]
=
\nabla_\theta \mathcal{L}_{\mathrm{RISE}}^{(e)}(\theta),
\]
where the weights are treated as fixed between refresh steps. Therefore, RISE can be interpreted as following stochastic gradients of a sequence of reweighted empirical risks that emphasize higher-loss quantile bins when $\varphi(\cdot)$ is increasing.
\end{proposition}

\begin{proof}
Fix a refresh step $e$ and condition on the weights $\{w_i^{(e)}\}$ and probabilities $\{p_i^{(e)}\}$. A single-sample stochastic gradient drawn by weighted sampling satisfies
\[
\mathbb{E}_{i\sim p^{(e)}}[\nabla_\theta \ell_i(\theta)]
=
\sum_{i=1}^N p_i^{(e)}\nabla_\theta \ell_i(\theta).
\]
Between refresh steps, the sampling weights are fixed with respect to the current optimization variable $\theta$. Hence
\[
\sum_{i=1}^N p_i^{(e)}\nabla_\theta \ell_i(\theta)
=
\nabla_\theta
\left(
\sum_{i=1}^N p_i^{(e)}\ell_i(\theta)
\right)
=
\nabla_\theta \mathcal{L}_{\mathrm{RISE}}^{(e)}(\theta).
\]
Substituting
\[
p_i^{(e)}
=
\frac{w_i^{(e)}}{\sum_{j=1}^N w_j^{(e)}}
\]
gives the normalized reweighted empirical risk. If $\varphi(\cdot)$ is increasing in the loss-quantile bin, higher-loss samples receive larger sampling probabilities, so the resulting stochastic gradients emphasize the tail of the empirical loss distribution.
\end{proof}

\section{Full Benchmark Results and Statistical Analyses} 
\subsection{Full Benchmark Results} \label{app:full_results}
Table~\ref{tab:regression_results} extends the summary in Table~\ref{tab:regression_results_compact} with per-model RMSE across all nine benchmarks. This table compares against classical baselines~\cite{pedregosa2011scikit}, tree-based ensembles~\cite{breiman2001random, prokhorenkova2018catboost, ke2017lightgbm, chen2016xgboost, geurts2006extratrees}, MLP-style architectures~\cite{gorishniy2021revisiting, gorishniy2022embeddings, gorishniy2024tabm, yan2023t2gformer}, transformer-based
models~\cite{huang2020tabtransformer, somepalli2021saint, song2019autoint, wang2021dcnv2, chen2023excelformer}, retrieval-based
methods~\cite{gorishniy2024tabr, ye2024modernnca, NODE, arik2021tabnet}, other deep tabular architectures \cite{klambauer2017snn, jeffares2023tangos,
ye2024ptarl, DANets, badirli2020grownet, nader2022dnnr, wu2024switchtab, marton2024grande}, and tabular foundation
models (either pretrained parametric models or non‑parametric inferences)~\cite{wang2022amformer, hollmann2022tabpfn, hollmann2025tabpfn, liu2023tabptm, margeloiu2024tabebm}. Symbols indicate runs that could not be completed due to time constraints~($\dag$), GPU memory limitations~($\ddag$), training divergence~($\S$), hard-coded dataset size restrictions~($\ast$), or prohibitive algorithm complexity~($\diamondsuit$). TabNSM achieves the lowest RMSE on seven of nine benchmarks, with competitive performance on CP where LightGBM remains the strongest baseline. 

\begin{table*}[t]
\centering
\caption{Evaluation of different models on regression tasks. Results report RMSE (lower is better).} \label{tab:regression_results}
\begin{threeparttable}
\resizebox{\textwidth}{!}{%
\begin{tabular}{llccccccccc}
\toprule
& \textbf{Model} & SA & CP & AQ & TO & SC & CD & ChD & EVP & RES \\
\midrule
\multirow{6}{*}{\rotatebox[origin=c]{90}{Classical}}
& LinearRegression     & 0.402          & 9.277         & 15.062    & 0.028    & 3.997      & 0.236 & 8717.544        & 35.156   & 3163 \\
& KNN-5                & 0.238         & 3.719         & 12.816    & 0.029    & 3.602      & 1.347 & 15860.523       & 26.104   & 3212 \\
& Lasso                & 0.402          & 9.277         & 15.064    & 0.028    & 3.964      & 0.251 & 8717.439        & 35.155   & 3163 \\
& Ridge                & 0.402          & 9.277         & 15.062    & \second{0.027}    & 3.982      & 0.236 & 8717.479        & 35.156   & 3163\\
& SVR-Linear           & 0.406         & 11.663        & 17.854    & 0.042    & 3.974     & $\diamondsuit, \dag$   & $\diamondsuit, \dag$    & $\diamondsuit, \dag$    & $\diamondsuit, \dag$ \\
& SVR-RBF              & 0.175         & 7.491         & 19.416   & 0.038    & 3.964       & $\ast$  & $\ast$            & $\ast$     & $\ast$ \\

\cmidrule(lr){2-11}
\multirow{5}{*}{\rotatebox[origin=c]{90}{Tree-based}}
& Random Forest  & 0.201 & 2.842 & 9.095 & \second{0.027} & 4.092 & 0.733 & 17525 & $\dag$ & $\dag$ \\
& CatBoost       & 0.171 & 2.739 & 7.692 & \second{0.027} & 3.971 & 0.009 & 19000 & 3.291 & 3144 \\
& LightGBM       & 0.174 & \second{2.680} & 7.902 & \second{0.027} & 4.639 & 0.043 & 18688 & 2.799 & \second{3135} \\
& XGBoost        & 0.197 & 2.690 & 11.583 & 0.028 & 4.116 & 0.040 & 18188 & 2.894 & 3186 \\
& Extratrees     & 0.182 & \best{2.617} & 5.869 & \second{0.027} & 3.455 & \best{0.004} & 22882 & 2.809 & 46267 \\
\cmidrule(lr){2-11}
\multirow{5}{*}{\rotatebox[origin=c]{90}{MLP-style}}
& MLP            & 0.181 & 2.939 & 8.967 & 0.028 & 3.920 & 0.046 & 4963 & 1.847 & 3161 \\
& ResNet         & 0.167 & 2.920 & 4.578 & 0.028 & 3.948 & 0.028 & 13022 & 1.556 & 3163 \\
& MLP-PLR        & 0.192 & 2.925 & 5.296 & 0.028 & 3.966 & 0.143 & 18476 & 3.222 & 40947 \\
& TabM           & 0.269 & 3.139 & 8.862 & 0.028 & 3.966 & 0.095 & 16032 & 2.101 & 40941 \\
& T2G-Former     & 0.192 & 2.937 & 5.358 & 0.028 & 3.962 & $\ddag$ & $\ddag$ & 1.867 & 21058 \\
\midrule
\multirow{6}{*}{\rotatebox[origin=c]{90}{Transformer}}
& TabTransformer & 0.298 & 8.605 & 4.189 & 0.028 & 3.866 & 0.046 & 7188 & 2.338 & 3936 \\  
& SAINT          & 0.156 & 2.901 & $\ddag$ & $\ddag$ & $\ddag$ & $\ddag$ & $\ddag$ & $\ddag$ & $\ddag$ \\
& FT-Transformer & 0.192 & 2.924 & 7.330 & 0.029 & 3.967 & 0.014 & 8940 & \second{1.507} & 3163 \\
& AutoInt        & 0.213 & 3.050 & 6.969 & 0.028 & 3.919 & 0.007 & 4999 & 2.587 & 3163 \\
& DCN-v2         & 0.165 & 8.133 & \second{4.081} & 0.038 & 3.751 & $\S$ & \second{4475} & 1.992 & 6667 \\
& ExcelFormer    & 0.165 & 2.901 & 8.088 & 0.028 & 3.667 & $\ddag$ & $\ddag$ & 1.608 & $\ddag$ \\
\cmidrule(lr){2-11}
\multirow{4}{*}{\rotatebox[origin=c]{90}{Retrieval}}
& TabR           & 0.150 & 2.867 & 5.062 & 0.029 & 3.967 & $\ddag$ & $\ddag$ & $\ddag$ & $\ddag$ \\
& ModernNCA      & 0.167 & 2.833 & 12.186 & 0.028 & 3.730 & $\ddag$ & $\ddag$ & $\ddag$ & $\ddag$ \\
& NODE           & 0.270 & 7.150 & 15.575 & 0.028 & 3.966 & 0.052 & 22661 & 5.157 & 3187 \\
& TabNet         & 0.279 & 3.346 & 8.615 & 0.028 & \second{3.454} & 0.284 & 24271 & 2.653 & 4152 \\
\midrule
\multirow{8}{*}{\rotatebox[origin=c]{90}{Other}}
& SNN            & 2.453 & 14.809 & 17.833 & 0.028 & 3.932 & 5.276 & 18759 & 71.000 & 3163 \\
& TANGOS         & 0.181 & 2.991 & 4.690 & 0.028 & 3.953  & 0.136 & 7298 & 3.219 & 4094 \\
& PTaRL          & 0.251 & 3.183 & 5.201 & 0.028 & 3.570  & 0.196 & 7865 & 4.157 & 4082 \\
& DANets         & 0.289 & 3.256 & 5.148 & 0.028 & 3.970  & 0.606 & 4808 & 9.367 & 41321 \\
& GrowNet        & 2.141 & 16.855 & 20.089 & 0.028 & 3.982 & $\S$ & 25264 & 13.448 & 124527 \\
& DNNR           & \second{0.149} & 3.065 & $\S$ & $\S$ & $\S$ & $\dag$ & $\dag$ & $\dag$ & $\dag$ \\
& SwitchTab      & 0.457 & 3.884 & 6.384 & 0.029 & 3.969  & 0.285 & 17978 & 6.227 & 4083 \\
& GRANDE         & 0.347 & 11.716 & 23.815 & 0.028 & 3.964 & 4.353 & 25440 & 69.129 & 25620 \\
\cmidrule(lr){2-11}
\multirow{3}{*}{\rotatebox[origin=c]{90}{\footnotesize Foundation}}
& AMFormer       & 0.205 & 2.980 & 16.363 & 0.028 & 3.964 & $\ddag$ & 17858 & $\ddag$ & 21149 \\
& TabEBM         & 0.178 & 2.740 & 6.748 & \second{0.027} & 3.522 & $\diamondsuit$ & $\diamondsuit$ & $\diamondsuit$ & $\diamondsuit$ \\
& TabPFN-v2      & 0.168 & 3.006 & 10.609 & \second{0.027} & 3.965 & $\ast$ & $\ast$ & $\ast$ & $\ast$ \\
& TabPTM         & 0.579 & 4.015 & 18.627 & 0.029 & 3.827 & 5.253 & 25630  & 65.471 & 4171 \\
\midrule
& \textbf{Ours}  & \best{0.137} & 2.842 & \best{3.999} & \best{0.021} & \best{3.322} & \second{0.007} & \best{4460} & \best{1.199} & \best{3103} \\
\bottomrule
\end{tabular}%
}
\begin{tablenotes}[flushleft]
\footnotesize
\item[] $\dag$ Not reported due to time constraints (24-hour limit); runs exceeding this duration were terminated.
\item[] $\ddag$ Not reported due to GPU memory limitations (GPU capacity: 79.21 GiB); models exceeding this requirement could not be executed.
\item[$\S$] Training diverged; the model produced losses orders of magnitude outside the comparable range and is treated as non-convergent.
\item[] $\ast$ Applicable only to limited numbers of features and samples (hard-coded).
\item[] $\diamondsuit$ Infeasible at dataset scale (algorithm complexity prohibitive).

\end{tablenotes}
\end{threeparttable}
\end{table*}

\subsection{Statistical Significance of Results} \label{app:stats}

We assess whether TabNSM's improvements over baselines are statistically significant using two complementary non-parametric tests~\cite{demsar2006statistical}.

\paragraph{Friedman Test}
We first test the global null hypothesis that all models perform equally across benchmarks. For each benchmark we assign integer ranks (1 = lowest RMSE = best), then apply the Friedman test to the resulting $k \times N$ rank matrix. We use the $k = 37$ models that completed all five benchmarks with reliable results (SA, CP, AQ, TO, SC), giving $N = 5$; the two excluded models (SAINT, DNNR) diverged or ran out of memory on three of the five. The Friedman statistic is $\chi^2(36) = 79.3$, $p = 4.25 \times 10^{-5}$, rejecting the null hypothesis and confirming that performance differences across models are not due to chance. Figure~\ref{fig:cd_diagram} shows the resulting Critical Difference diagram.

\begin{figure}[ht]
\centering
\includegraphics[width=0.95\linewidth]{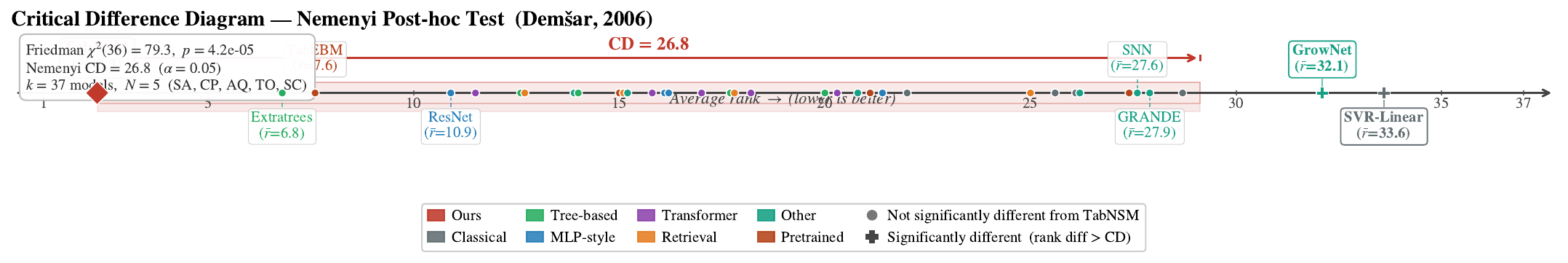}
\caption{Critical Difference diagram (Nemenyi post-hoc test, $\alpha = 0.05$). Each dot is a model coloured by architectural family; its horizontal position is its average rank across five benchmarks (lower = better). The CD bracket shows the minimum rank difference required for statistical significance (CD = 26.8). Models marked with $+$ (GrowNet, SVR-Linear) fall outside TabNSM's critical difference band and are significantly worse. The large CD is expected when $k \gg N$; the Friedman test nonetheless confirms a significant overall difference ($p = 4.25\times10^{-5}$).}
\label{fig:cd_diagram}
\end{figure}

The Nemenyi post-hoc test is known to be conservative when $k \gg N$: with 37 models and only 5 benchmarks, the critical difference is CD = 26.8 rank units, leaving most pairwise comparisons inconclusive regardless of the true effect. This is a known limitation of the Friedman--Nemenyi procedure in wide-but-short rank matrices. The Wilcoxon signed-rank test sidesteps this by testing each pair directly across all nine benchmarks, providing substantially more statistical power. The two tests are therefore not contradictory: Nemenyi confirms only the largest rank gaps with high confidence given $N=5$, while Wilcoxon leverages all nine benchmarks to establish significance for all baselines.

\paragraph{Pairwise Wilcoxon Signed-Rank Test}
Following the rejection of the global null, we conduct 38 one-sided pairwise Wilcoxon signed-rank tests comparing TabNSM against each baseline (H$_1$: TabNSM rank $<$ baseline rank) across all nine benchmarks, imputing missing values as worst rank. $p$-values are corrected for multiple comparisons using the Benjamini--Hochberg procedure~\citep{benjamini1995controlling} at $\alpha = 0.05$. TabNSM achieves a significantly lower rank than every baseline after correction (all corrected $p \leq 0.02$). The closest competitors by average rank are ResNet ($\bar{r} = 9.72$, $p = 0.002$), Extratrees ($\bar{r} = 11.44$, $p = 0.020$), MLP ($\bar{r} = 11.78$, $p = 0.002$), CatBoost ($\bar{r} = 12.44$, $p = 0.010$), and LightGBM ($\bar{r} = 12.78$, $p = 0.010$), compared to TabNSM's average rank of $\bar{r} = 1.89$ across nine benchmarks. Figure~\ref{fig:wilcoxon} summarises the average ranks and significance results for all baselines. 

We note that worst-rank imputation is conservative for TabNSM (which has no missing runs) and liberal for baselines with many incomplete runs; results should therefore be interpreted as an aggregate trend rather than a precise pairwise comparison for models with many incomplete entries.

\begin{figure}[ht]
\centering
\includegraphics[width=0.95\linewidth]{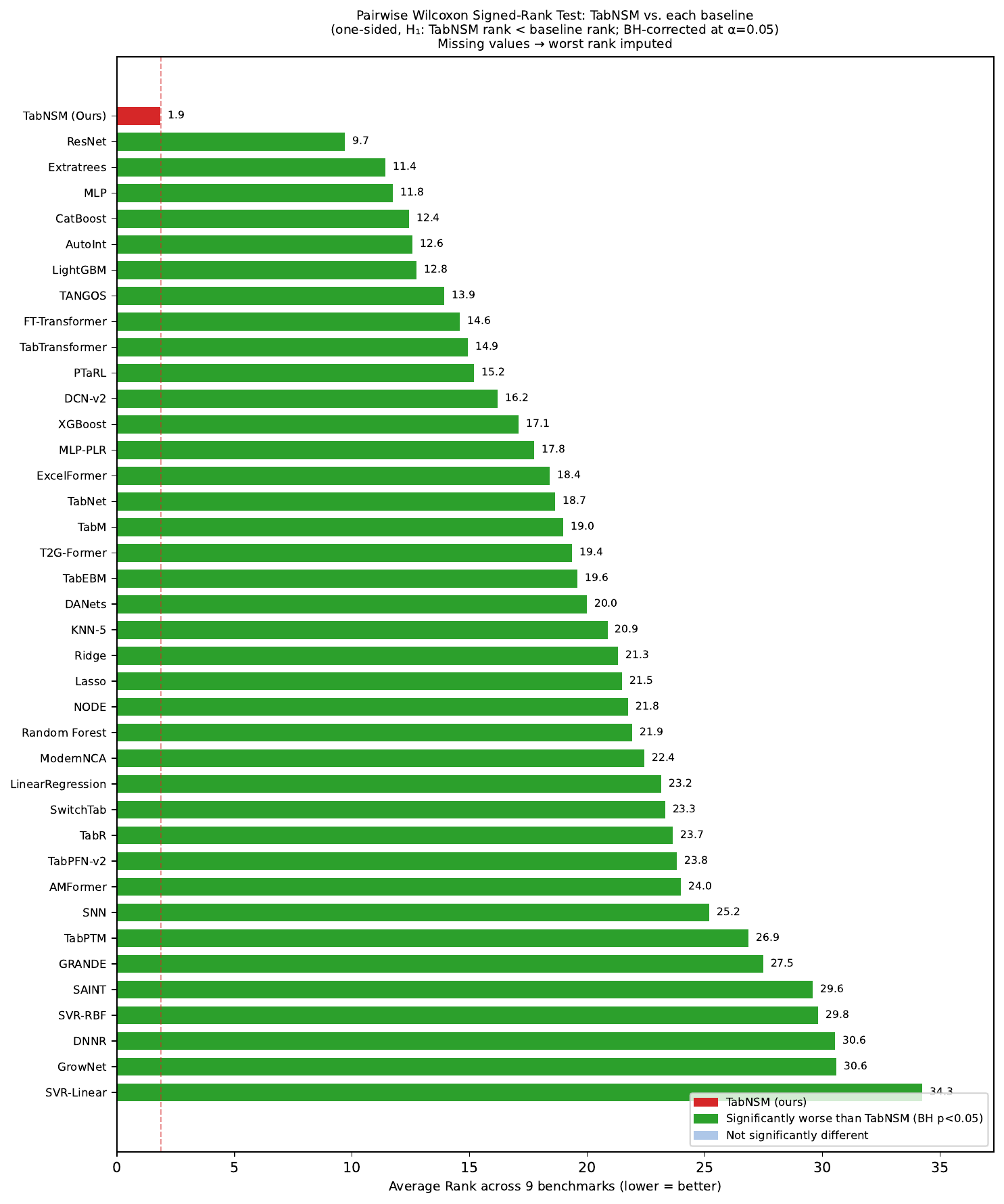}
\caption{Average rank across nine benchmarks for each model (lower is better), coloured by architectural family. All baselines are statistically significantly worse than TabNSM (one-sided Wilcoxon signed-rank test, Benjamini--Hochberg corrected~\citep{benjamini1995controlling}, $\alpha = 0.05$; all corrected $p \leq 0.02$). Missing values are imputed as worst rank.}
\label{fig:wilcoxon}
\end{figure}
\section{Brief Summary of Baseline Model Setup} \label{app:baselines}

All baseline models were implemented in Python using standard libraries. Data were loaded from ARFF files into pandas DataFrames, categorical strings decoded to UTF-8 and converted to numeric where possible, missing values filled with zeros, and the target variable removed before training. The data were split into 80\% training+validation and 20\% test, with the training+validation portion further split into 90\% training and 10\% validation. Neural and hybrid models (DeepGBM, TabNet, FT-Transformer) used z-score normalization, while tree-based models used raw values.

Performance was evaluated on the held-out test set using RMSE (primary), MSE, MAE, and $R^2$. Hyperparameter optimization was performed with Optuna (3-fold CV on training+validation, 20 trials, maximizing negative RMSE). The best configuration was used to train a final model on the training set (using the validation set when supported) and evaluated on the held-out test set.

The code for models marked with $^*$ was not provided for the regression task. Models marked with $\star$ were not tuned for most large-dataset benchmarks, often resulting in non-convergence or outlier performance. Therefore, we do not report their results in Table~\ref{tab:regression_results}.

\section{Instance-adaptive feature selectivity} \label{app:interpretability}

To examine whether TabNSM learns instance-adaptive feature selectivity, we visualize per-instance feature attributions on the test set. Figure~\ref{fig:interpretability} shows a feature attribution heatmap, where each row corresponds to a test instance and each column corresponds to an input feature. The attribution score is computed using gradient--embedding attribution and normalized to $[0,1]$ within each instance. Brighter values indicate features with stronger influence on the model's regression output for that sample.

The heatmap reveals two patterns. First, some features appear as consistently bright columns across many instances, suggesting globally influential predictors. Second, many instances exhibit distinct sparse attribution patterns, indicating that TabNSM dynamically emphasizes different feature subsets depending on the input context. This behavior is consistent with the design of the Adaptive Sparse Interaction Module (ASIM), which encourages instance-adaptive sparse feature interaction rather than relying on a fixed global feature weighting.

\begin{figure}[tbh]
  \centering
  \includegraphics[height=6cm, width=0.85\linewidth]{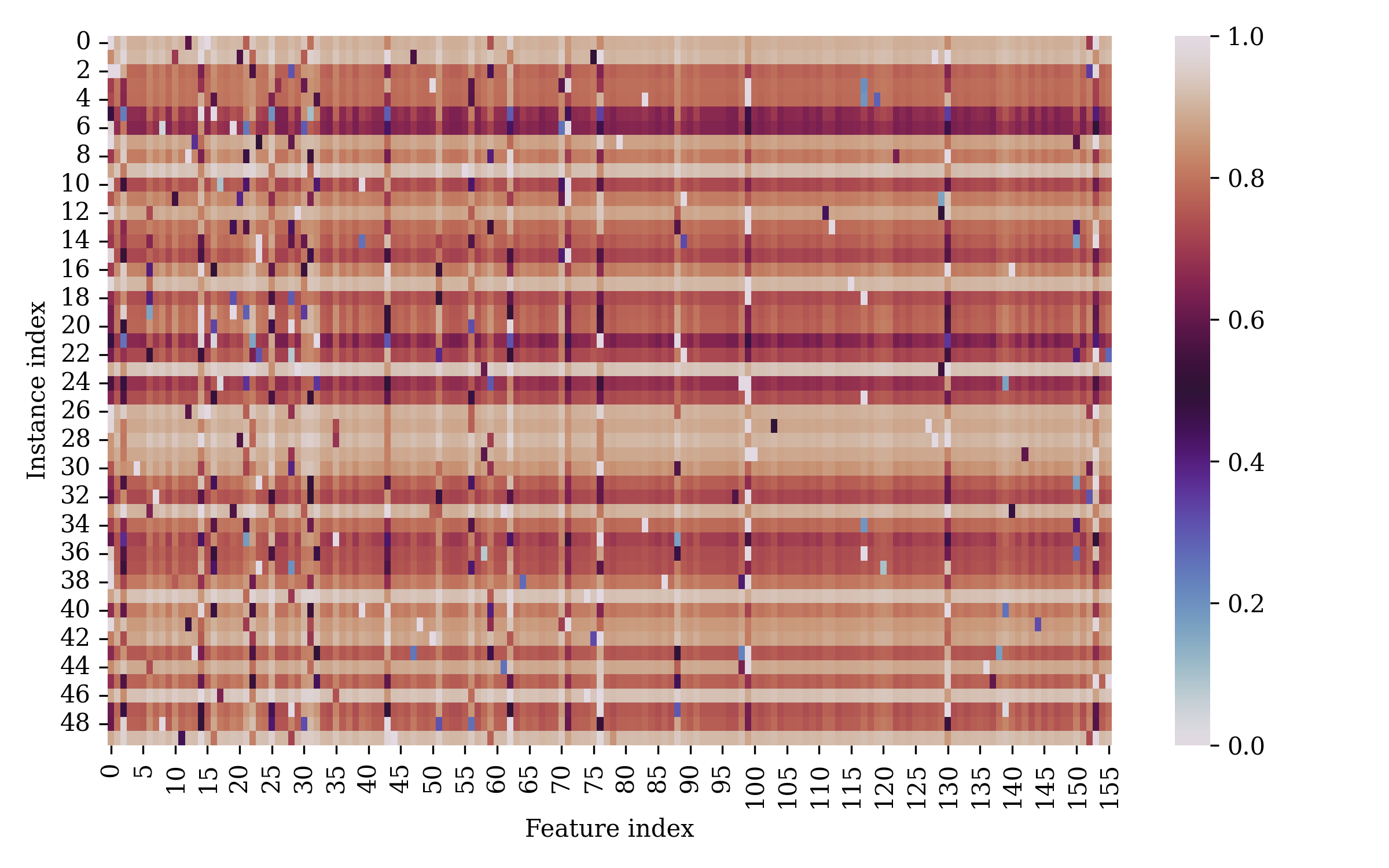}
  \caption{\textbf{Per-instance feature attribution heatmap for TabNSM.}
  Each row corresponds to a test instance and each column corresponds to an input feature. Cell intensity represents the normalized attribution score, computed using gradient--embedding attribution and scaled to $[0,1]$ within each instance. Brighter cells indicate features with stronger influence on the prediction. The sparse and instance-varying attribution patterns suggest that TabNSM adaptively focuses on different feature subsets across samples.}
  \label{fig:interpretability}
\end{figure}
\section{Dataset Descriptions} \label{app:datasets}

This section provides detailed descriptions of the datasets used in our
experiments, along with the associated prediction tasks.

\begin{table}[tbh]
    \centering
    \caption{Dataset Statistics}
    \label{table:datasets}
    \resizebox{0.7\textwidth}{!}{
    \begin{tabular}{lccc}
    \toprule
    \textbf{Dataset} & \textbf{Abbreviation} & \textbf{\# Samples} & \textbf{\# Features} \\
    \midrule
    Sarcos  & SA & 48,933 & 27 \\
    CPU     & CP & 8,192 & 12 \\
    Topo    & TO & 8,885 & 266 \\
    SCAN    & SC & 4,355 &  386 \\
    Crime Data & CD & 1,004,991 & 401 \\
    U.S. Chronic Disease & ChD & 309,215 & 405 \\
    Electric Vehicle Population Data & EVP & 264,628 & 392 \\
    Air Quality & AQ & 18,862 & 390 \\
    Real Estate Sales & RES & 1,048,575 & 389 \\
    \bottomrule
    \end{tabular}}
\end{table}

\paragraph{ topo\_2\_1.}
The Topo dataset was used, with 1,143 molecular descriptors computed using the Adriana. Molecular structures and response values were taken directly from the original studies. The final attribute in each file was used as the prediction target; for the Topo dataset, this target corresponds to oz267.

\paragraph{Sarcos Robot Arm Dataset.}
The Sarcos dataset consists of measurements from a 7-degree-of-freedom
robotic arm performing various motions. Each instance includes joint
positions, velocities, and accelerations, and the prediction task is to
estimate the torque of a specific joint (V22). This dataset is commonly
used to evaluate regression models for learning inverse dynamics.

\paragraph{Electric Vehicle Population Dataset.}
This dataset contains records of battery electric and plug-in hybrid
vehicle registrations in Washington State. Features include vehicle
attributes and registration information, and the task is to predict the
electric driving range.

\paragraph{Air Quality}
The dataset consists of New York City air quality surveillance indicators measured across neighborhoods and time. Features describe pollutant exposure and related environmental health metrics, while the prediction target, Data Value, represents the quantitative measurement of each air quality indicator used for supervised learning.

\paragraph{U.S. Chronic Disease Dataset.}
The Chronic Disease dataset provides standardized public health
indicators collected across U.S. states and territories. The prediction
task involves estimating numeric values associated with chronic disease
metrics.

\paragraph{Crime Dataset.}
The Crime dataset consists of crime incident reports from Los Angeles
since 2020. Each record includes temporal, categorical, and spatial
features, and the task is to predict the geographic area associated with
each incident.

\paragraph{Real Estate Sales Dataset.}
This dataset contains property transaction records reported by
municipalities in Connecticut. Given property and transaction features,
the task is to predict the sales ratio, a commonly used indicator in real
estate valuation and property tax assessment.

\paragraph{Standardized Centralized Alzheimer’s Neuroimaging (SCAN) Dataset.}
The SC benchmark was constructed from controlled-access multimodal data
obtained through the NACC data-request process under the NACC Data Use
Agreement. The data include neuroimaging, fluid biomarkers, and demographic
variables collected across Alzheimer's Disease Research Centers. The
prediction task is to estimate Mini-Mental State Examination (MMSE) scores.

\subsection{Dataset Links}
\label{links}
Eight benchmark datasets are publicly available; NACC/SCAN data are available
to researchers through NACC's data-request process under the NACC Data Use
Agreement. Table~\ref{tab:benchmark_datasets} provides the public source or
access page for each dataset.

\begin{table}[tbh]
\centering
\caption{Benchmark dataset sources.}
\label{tab:benchmark_datasets}
\begin{threeparttable}
\small
\begin{tabularx}{0.95\textwidth}{l>{\raggedright\arraybackslash}X}
\toprule
Dataset & URL \\
\midrule
CPU & \url{https://www.openml.org/search?type=data&status=active&id=42841} \\
Topo & \url{https://www.openml.org/search?type=data&status=active&id=422} \\
Sarcos Robot Arm Data & \url{https://www.openml.org/search?type=data&status=active&id=43873&sort=runs} \\
Electric Vehicle Population Data & \url{https://catalog.data.gov/dataset/electric-vehicle-population-data} \\
U.S. Chronic Disease Data & \url{https://catalog.data.gov/dataset/u-s-chronic-disease-indicators} \\
Air Quality & \url{https://catalog.data.gov/dataset/air-quality} \\
Crime Data from 2020 to Present & \url{https://catalog.data.gov/dataset/crime-data-from-2020-to-present} \\
Real Estate Sales Data & \url{https://catalog.data.gov/dataset/real-estate-sales-2001-2018} \\
Standardized Centralized Alzheimer's Neuroimaging (SCAN) Data & \url{https://scan.naccdata.org/} \\
\bottomrule
\end{tabularx}
\begin{tablenotes}[flushleft]
\footnotesize
\item[] Note: All datasets were accessed by August 2025; NACC/SCAN data require an approved request and Data Use Agreement.
\end{tablenotes}
\end{threeparttable}
\end{table}

\section{Seed Variation}
\label{appendix:seed}
To assess the robustness of the proposed approach with respect to random initialization and training stochasticity, we repeat each experiment 20 times using different random seeds. Figure~\ref{fig:seed} reports the distribution of test Root Mean Squared Error (RMSE) across runs using boxplots. The results show that TabNSM maintains stable performance across seeds, with relatively small interquartile ranges and limited variation across runs. This suggests that the reported performance is not driven by a favorable initialization, but reflects consistent behavior under repeated training. Overall, the limited variance indicates that the proposed method is robust to randomness introduced during training.
\begin{figure}[tbh]
  \centering
  \includegraphics[width=0.5\textwidth]{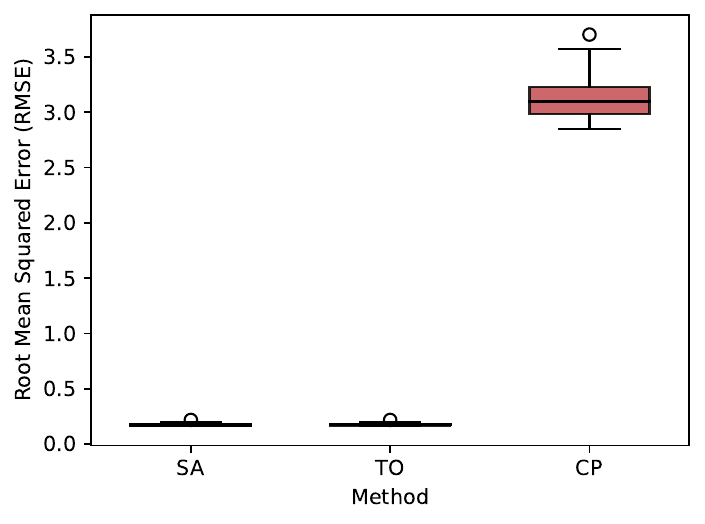}
  \caption{\textbf{Sensitivity analysis with respect to random seed initialization.}
  The boxplot shows the distribution of test Root Mean Squared Error (RMSE) over 20 independent runs with different random seeds. Boxes indicate the interquartile range (IQR), the center line denotes the median, and whiskers extend to $1.5\times\mathrm{IQR}$. Lower values indicate better performance.}
  \label{fig:seed}
\end{figure}
\section{Efficiency and Memory Footprint}
\label{Memory}

Figure~\ref{fig:sampling} and Table~\ref{tab:catboost_tabnsm_efficiency} compare the practical training time of CatBoost and TabNSM across datasets with varying sample sizes $n$ and feature dimensions $d$. Runtime is measured under a fixed validation-loss threshold protocol, where training time is recorded until each method reaches a predefined validation-loss level. This threshold-based criterion provides a practical comparison of convergence speed across models with different optimization behavior, although the methods use different computational backends: CatBoost is evaluated on CPU, whereas TabNSM uses GPU acceleration.

As shown in Table~\ref{tab:catboost_tabnsm_efficiency}, CatBoost runtime increases as dataset size and feature dimensionality grow, consistent with the scaling behavior of gradient-boosted decision trees on large tabular datasets. In contrast, TabNSM maintains low runtime across the evaluated settings. This reflects the mini-batch training procedure and the use of sparse feature interaction modeling, which avoids dense all-pair feature mixing and limits the dependence on feature dimensionality. These results suggest that TabNSM provides favorable practical efficiency in large-scale, high-dimensional regression settings.

Tables~\ref{tab:gpu_peak_mem} and~\ref{tab:cpu_mem_increase} further report memory usage on representative large datasets. Since CatBoost is run on CPU, GPU memory is reported only for TabNSM. CPU memory increase is reported for both methods. TabNSM shows moderate GPU memory usage and lower CPU memory increase than CatBoost on the evaluated large datasets. Overall, the results highlight complementary scalability profiles: CatBoost remains a strong CPU-based tree ensemble, while TabNSM offers efficient GPU-accelerated training with favorable memory behavior on large, high-dimensional datasets.

\begin{table}[t]
\centering
\small
\caption{Runtime in seconds for CatBoost and TabNSM under the validation-loss threshold protocol. Lower is better. CatBoost is evaluated on CPU, whereas TabNSM uses GPU acceleration.}
\label{tab:catboost_tabnsm_efficiency}
\begin{tabular}{lrr|rr}
\toprule
Dataset & $n$ & $d$ & CatBoost Time (s) & TabNSM Time (s) \\
\midrule
cpu    & 8,192     & 12  & 1.042  & \textbf{0.346}  \\
topo   & 8,885     & 266 & 10.788 & \textbf{0.333}  \\
sarcos & 48,933    & 27  & 2.111  & \textbf{0.333}  \\
EVP    & 264,628   & 392 & 14.986 & \textbf{0.346}  \\
ChD    & 309,215   & 405 & 17.252 & \textbf{0.349}  \\
CD     & 1,004,991 & 401 & 45.476 & \textbf{0.398}  \\
\bottomrule
\end{tabular}
\end{table}

\begin{table}[t]
\centering
\small
\caption{Peak GPU memory usage in MB for TabNSM. CatBoost is run on CPU, so GPU memory is not applicable.}
\label{tab:gpu_peak_mem}
\begin{tabular}{lcc}
\toprule
\textbf{Dataset} & \textbf{CatBoost} & \textbf{TabNSM} \\
\midrule
EVP & N/A & 477 \\
ChD & N/A & 529 \\
CD  & N/A & 1621 \\
\bottomrule
\end{tabular}
\end{table}

\begin{table}[t]
\centering
\small
\caption{CPU memory increase in MB during training. Lower is better.}
\label{tab:cpu_mem_increase}
\begin{threeparttable}
\begin{tabular}{lcc}
\toprule
\textbf{Dataset} & \textbf{CatBoost} & \textbf{TabNSM} \\
\midrule
EVP & 3763.77  & \textbf{3006.73} \\
ChD & 4325.94  & \textbf{3408.70} \\
CD  & 12080.61 & \textbf{8524.16} \\
\bottomrule
\end{tabular}
\begin{tablenotes}[flushleft]
\footnotesize
\item[] These large-scale settings are not evaluated with TabPFN because of its configuration limits on the number of samples and features.
\end{tablenotes}
\end{threeparttable}
\end{table}

\begin{figure}[tbh]
  \centering
  \includegraphics[width=0.65\textwidth]{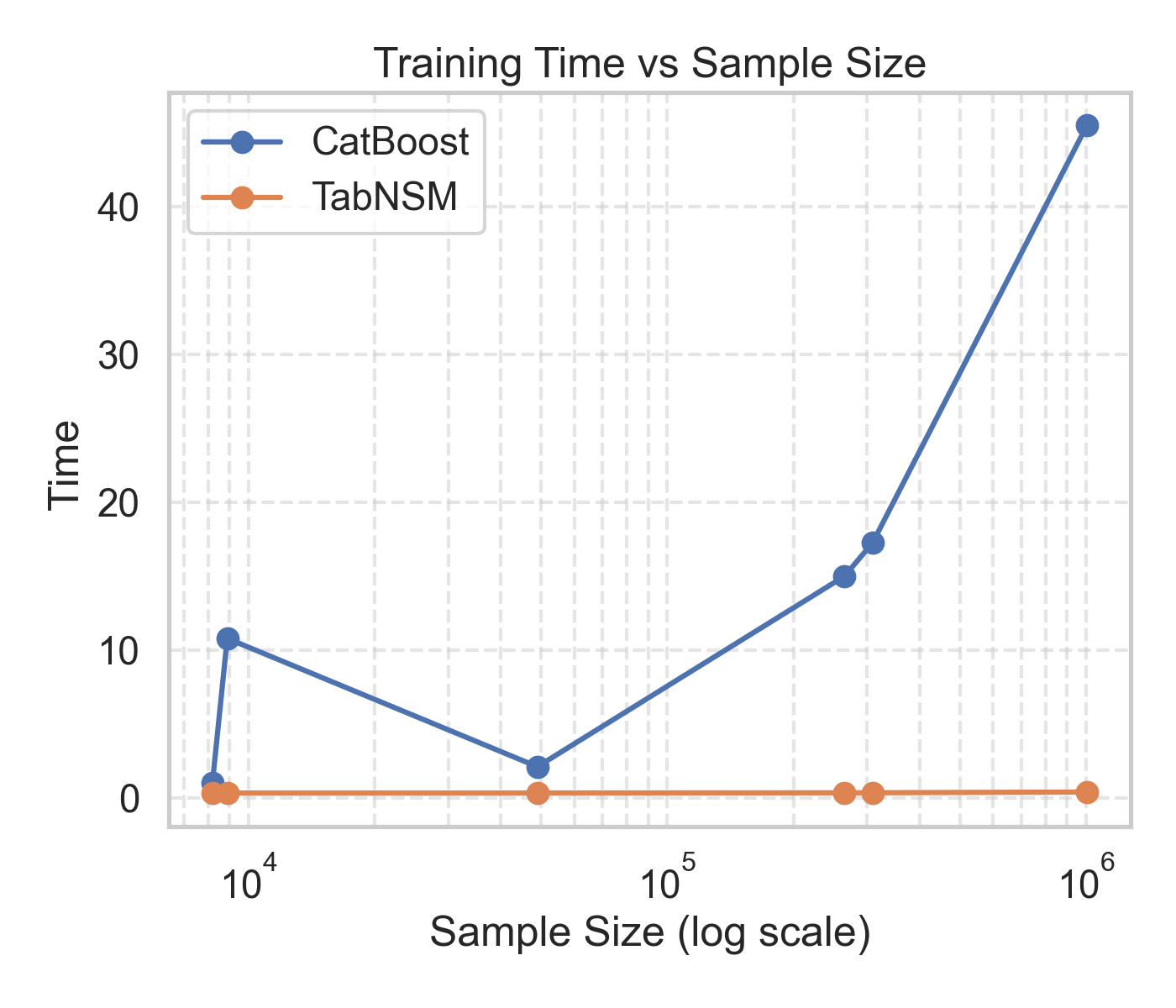}
  \caption{\textbf{Training time versus sample size.}
  Runtime is measured under the validation-loss threshold protocol. CatBoost runtime increases with dataset size, whereas TabNSM maintains low runtime across the evaluated settings, reflecting mini-batch GPU training and sparse feature interaction modeling.}
  \label{fig:sampling}
\end{figure}

Additional runtime and memory measurements are provided in Appendix~\ref{Memory}; Figure~\ref{fig:bubble_chart} summarises the accuracy--efficiency trade-off across all evaluated methods.
\begin{figure}[tbh]
  \centering
  \includegraphics[width=\textwidth]{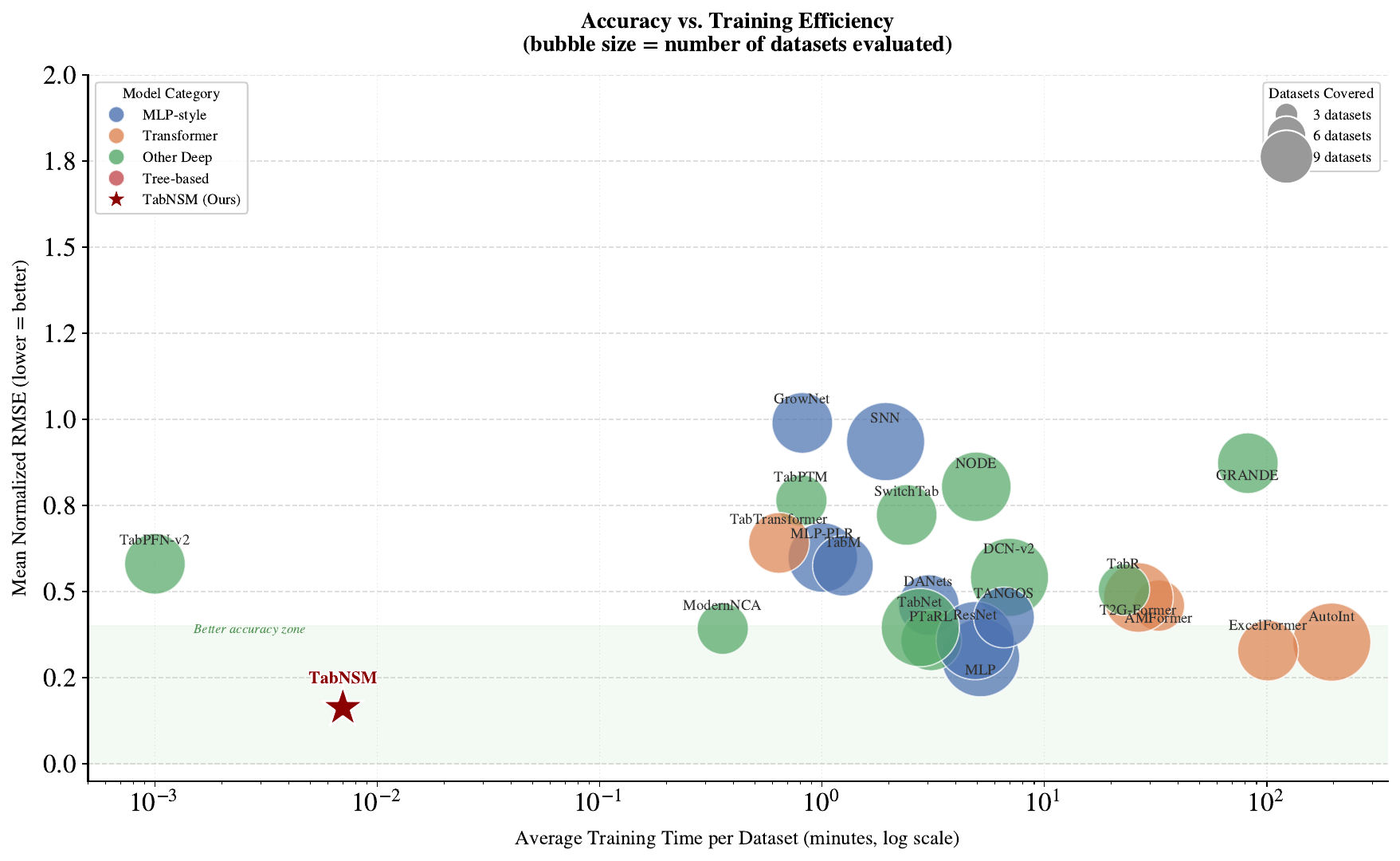}
  \caption{\textbf{Accuracy--efficiency trade-off for tabular learning methods.}
  The figure summarizes predictive accuracy and computational efficiency across representative tabular learning models.}\label{fig:bubble_chart}
\end{figure}


\section{Sparse Feature-wise Attention: Full Formulation} \label{app:sparse_attn}

Within ASIM, we treat the $D$ input features as a sequence of token embeddings $X_{\mathrm{tok}} \in \mathbb{R}^{D \times d_{\mathrm{tok}}}$ and apply a structured sparse self-attention operator $\mathcal{A}_{\mathrm{sparse}}$~\citep{eslamian2025tabnsa, NSA}. Rather than computing all $D^2$ pairwise feature interactions, each feature token attends only to a compact, instance-adaptive subset of other features. This subset is constructed from three complementary components designed to jointly cover local, global, and instance-specific interaction patterns.

\paragraph{Local neighborhood.}
Features adjacent in the input ordering often encode correlated measurements (e.g., related embedding dimensions or grouped sensor readings). For feature $d$, we retain the $s$ nearest neighbors as a sliding window:
\begin{equation}
\mathcal{K}^{\mathrm{win}}_d = \mathbf{k}_{d-s:d+s}, \quad \mathcal{V}^{\mathrm{win}}_d = \mathbf{v}_{d-s:d+s}.
\end{equation}

\paragraph{Compressed global summary.}
To give each query feature coarse awareness of the full feature set, we partition the $D$ features into blocks of size $b$ (stride $r \le b$) and compress each block to a single representative key--value pair via a learnable MLP $\phi$:
\begin{equation}
\mathcal{K}^{\mathrm{cmp}}_d = \left\{\phi\!\left(\mathbf{k}_{ib+1:\,ib+b}\right) \;\middle|\; 0 \le i \le \left\lfloor\tfrac{D-b}{r}\right\rfloor\right\}, \label{eq:compress}
\end{equation}
with an analogous $\mathcal{V}^{\mathrm{cmp}}_d$. This reduces global context to $\mathcal{O}(D/r)$ keys.

\paragraph{Instance-adaptive block selection.}
The compressed attention scores in Eq.~\eqref{eq:compress} provide a low-cost proxy for identifying which feature groups are most relevant to the current query. We compute:
\begin{equation}
\mathbf{p}_d^{\mathrm{slc}} = \operatorname{Softmax}\!\left(\mathbf{q}_d^\top \mathcal{K}^{\mathrm{cmp}}_d\right),
\end{equation}
rank all blocks by $\mathbf{p}_d^{\mathrm{slc}}$, and retrieve the original full-resolution keys and values for the top-$m$ blocks:
\begin{equation}
\mathcal{I}_d = \{i \mid \mathrm{rank}(\mathbf{p}_d^{\mathrm{slc}}[i]) \le m\}, \quad
\mathcal{K}^{\mathrm{slc}}_d = \operatorname{Cat}\bigl[\{\mathbf{k}_{ib+1:(i+1)b} \mid i \in \mathcal{I}_d\}\bigr].
\end{equation}
This recovers fine-grained interactions for the most relevant feature groups with a fixed budget of $mb$ keys per query.

\paragraph{Gated aggregation.}
The three branch outputs are merged via input-dependent gating scalars $g_d^c \in [0,1]$, produced by a small MLP with sigmoid activation applied to the query embedding:
\begin{equation}
\tilde{X}_d = \sum_{c \in \{\mathrm{win},\,\mathrm{cmp},\,\mathrm{slc}\}} g_d^c \cdot \operatorname{Attn}\!\left(\mathbf{q}_d,\, \mathcal{K}^c_d,\, \mathcal{V}^c_d\right).
\end{equation}
The total attended context per feature is $N = 2s + \lfloor(D-b)/r\rfloor + mb \ll D$, giving overall complexity $\mathcal{O}(D\,d_{\mathrm{tok}}\,N)$, sub-quadratic in $D$ for fixed sparse hyperparameters $(s,b,r,m)$.

\section{Feature-Token Mixing (FTM): Full Formulation} \label{app:tabmixer}

Adapting the dual-axis mixing block introduced in TabMixer~\citep{eslamian2025tabmixer}, the FTM operator $\mathcal{M}(\cdot)$ acts on the residual sparse representation $U=w_1X_{\mathrm{tok}}+\tilde{X}\in\mathbb{R}^{D\times d_{\mathrm{tok}}}$, where $D$ feature tokens each carry a $d_{\mathrm{tok}}$-dimensional embedding.

The operator applies two parallel mixing branches whose outputs are fused by element-wise gating:
\begin{align}
U_1 &= \mathrm{GeLU}\!\left(W_2\bigl[\mathrm{ReLU}(W_1\,\mathrm{LN}(U)^\top + b_1)\bigr] + b_2\right)^\top \label{eq:tm1} \\
U_2 &= W_4\bigl(\mathrm{ReLU}(W_3\cdot\mathrm{LN}(U) + b_3)\bigr) + b_4 \label{eq:tm2} \\
U_3 &= \mathrm{SiLU}\!\left(\mathrm{LN}(U_1 \odot U_2)\right) \label{eq:tm3} \\
Z   &= U + U_3, \label{eq:tm4}
\end{align}
where $W_1,W_2 \in \mathbb{R}^{d_{\mathrm{tok}}\times d_{\mathrm{tok}}}$ and $W_3,W_4 \in \mathbb{R}^{B_f \times B_f}$ are learnable weights, with $B_f$ a constant feature mini-batch size ($B_f \ll D$); $\mathrm{LN}(\cdot)$ denotes Layer Normalization; and $\odot$ is element-wise multiplication.

The two branches serve distinct mixing roles. The \emph{embedding-wise} branch (Eq.~\eqref{eq:tm1}) transposes the input so that the MLP operates across the $d_{\mathrm{tok}}$ embedding dimensions for each feature independently, capturing intra-token structure. The \emph{feature-wise} branch (Eq.~\eqref{eq:tm2}) mixes along the token axis. To avoid a quadratic cost in $D$, we process the $D$ feature tokens in disjoint mini-batches of size $B_f$ (padding the last batch if needed). The weight matrices $W_3,W_4$ are applied per batch, leading to a total complexity of $\mathcal{O}(D \cdot B_f)$ for this branch---linear in $D$. In our experiments, $B_f=64$ and the number of feature tokens $D$ ranges from $128$ to $512$, so the feature-wise mixing overhead remains small.

The outputs of the two branches are combined via element-wise multiplication, followed by Layer Normalization and a SiLU activation (Eq.~\eqref{eq:tm3}). A residual connection adds the result back to the original representation (Eq.~\eqref{eq:tm4}), ensuring stable gradient flow while enriching the representation with globally mixed context.

\section{Additional Ablation Studies and Sensitivity Analyses}
\label{appendix:additional_ablation}

\subsection{Ablation of the FTM Pathway} \label{ablation:tabmixer}

We ablate the FTM pathway by removing it while keeping the sparse attention backbone and training protocol fixed. Table~\ref{tab:tabmixer-ablation} shows that ablating FTM degrades RMSE on CP, TO, SA, and EVP, with particularly large drops on EVP (570.09 vs. 1.199). This indicates that the adapted FTM pathway contributes to the overall model performance, while sparse selection remains the primary architectural performance driver.

\begin{table}[tbh]
\centering
\caption{Ablation of the FTM pathway. Lower RMSE is better.}
\label{tab:tabmixer-ablation}
\begin{tabular}{lcccc}
\toprule
 & CP & TO & SA & EVP \\
\midrule
w/o FTM & 3.966 & 0.034 & 0.176 & 570.09 \\
Full TabNSM & \textbf{2.842} & \textbf{0.021} & \textbf{0.137} & \textbf{1.199} \\
\bottomrule
\end{tabular}
\end{table}

\subsection{Ablation on Loss Function}
\label{ablation:loss_function}

We analyze the effect of different regression losses to justify the proposed GridLoss and its use in a blended objective. We compare standard pointwise losses, including MSE, MAE, and Huber, a logarithmic robust variant, LogCosh, and the proposed GridLoss across three representative datasets with different scales and noise characteristics. Figure~\ref{fig:ablation_lossfunction} reports RMSE on a log scale, with relative changes annotated with respect to GridLoss.

\begin{figure}[tbh]
  \centering
  \includegraphics[height=4.5cm, width=0.6\linewidth]{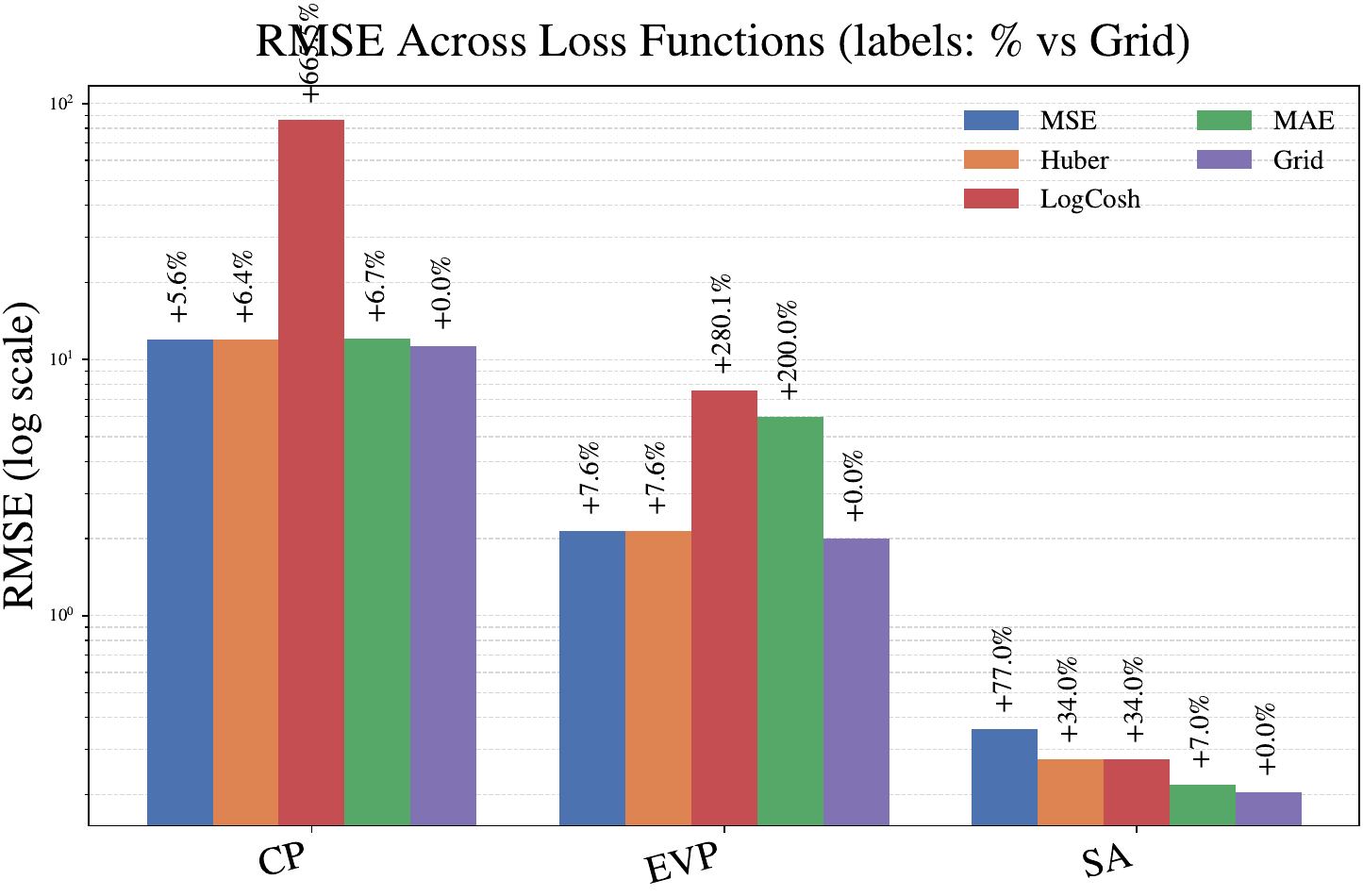}
  \caption{\textbf{Loss-function ablation across three datasets.}
  RMSE is reported on a log scale. Values above bars indicate relative change with respect to GridLoss. GridLoss achieves the lowest or comparable RMSE across the evaluated datasets, while pointwise losses such as MAE and LogCosh can degrade substantially in high-error regimes.}
  \label{fig:ablation_lossfunction}
\end{figure}

Across the evaluated benchmarks, GridLoss achieves competitive or superior RMSE relative to standard regression losses. Compared with MAE and LogCosh, GridLoss is less sensitive to scale variation and large residuals, while remaining competitive with Huber loss. The gains are most pronounced on datasets with heterogeneous error distributions, where purely pointwise losses may not adequately capture ordinal alignment along the target scale.

These results motivate the use of a blended objective in the full model,
\[
\mathcal{L}
=
\alpha \mathcal{L}_{\mathrm{Huber}}
+
(1-\alpha)\mathcal{L}_{\mathrm{Grid}},
\]
which combines the local robustness of Huber loss with the structure-aware ordinal supervision of GridLoss. Empirically, this combination yields stable optimization and strong predictive performance, particularly in high-dimensional regression settings. We fix $\alpha$ across all experiments and observe limited sensitivity to this choice; additional analysis is provided in Appendix~\ref{app:gridloss_consistency}.

\subsection{Ablation on effect of RISE}
 We evaluate the impact of RISE by disabling it and reverting to uniform sampling. As shown in Figure~\ref{fig:final_val_boxplot}, removing RISE increases both the mean validation error and the run-to-run variance. When RISE is enabled, validation performance is more consistent across runs, indicating improved optimization stability. Although RISE introduces additional sampling overhead, this cost is modest relative to the observed gains in validation performance. The box plot summarizes the distribution of final validation metrics across multiple random seeds, showing that RISE achieves a lower median error and reduced variability compared to uniform sampling.

\begin{figure}[t]
    \centering
    \begin{minipage}[b]{0.49\textwidth}
        \centering
        \includegraphics[height=5cm, width=5cm]{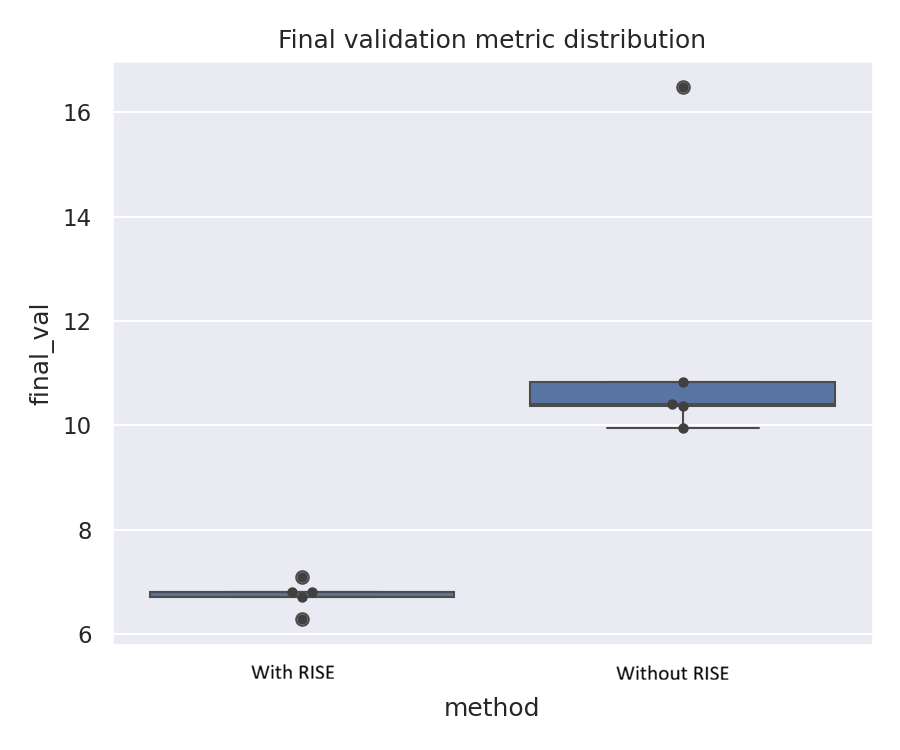}
        \caption{Final validation metric distributions across multiple runs for models trained with and without RISE.}
        \label{fig:final_val_boxplot}
    \end{minipage}
\end{figure}

\subsection{Sensitivity to Feature Dimensionality and Sample Size}
\label{appendix:dimension}

We evaluate robustness to changing feature dimensionality and training-set size. For feature-dimensionality analysis, we vary the \emph{feature fraction}, defined as the proportion of retained input features for each dataset, while keeping the number of rows fixed. For each fraction, we subsample feature columns using a fixed random seed and keep all other training settings identical to the main experimental protocol. For sample-size analysis, we vary the fraction of training samples while keeping the full feature set fixed.

Figure~\ref{fig:dimension_sample} reports MSE for TabNSM and baseline models on the TO and SC datasets. In the feature-fraction setting, TabNSM maintains stable performance across increasing dimensionality, whereas the baseline exhibits higher error and more pronounced non-monotonic behavior. In the sample-fraction setting, TabNSM remains competitive under reduced data availability. These trends suggest that TabNSM is robust to both high-dimensional feature spaces and limited training data. GridLoss and RISE further contribute to stability by providing structure-aware gradients and emphasizing difficult examples during training.

\begin{figure}[t]
  \centering
  \begin{subfigure}{0.7\textwidth}
    \centering
    \includegraphics[width=\linewidth]{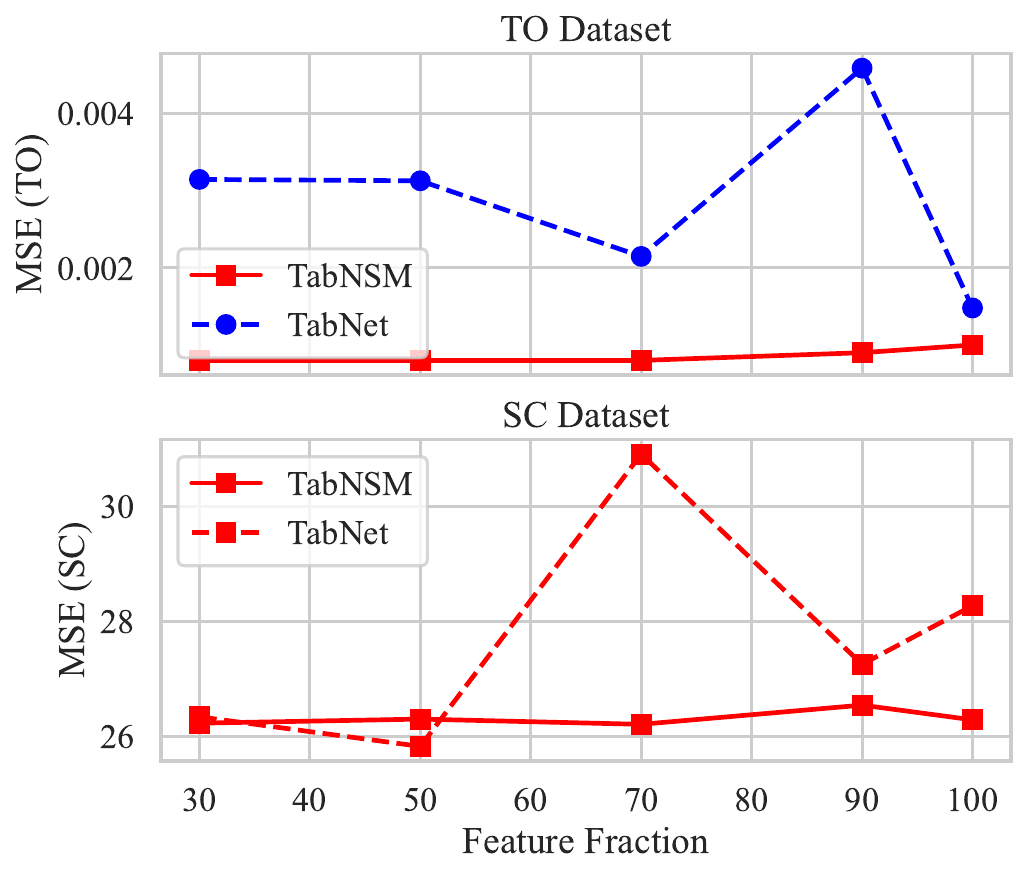}
    \caption{MSE versus feature fraction on the TO and SC datasets.}
    \label{fig:dimension}
  \end{subfigure}
  \begin{subfigure}{0.7\textwidth}
    \centering
    \includegraphics[width=\linewidth]{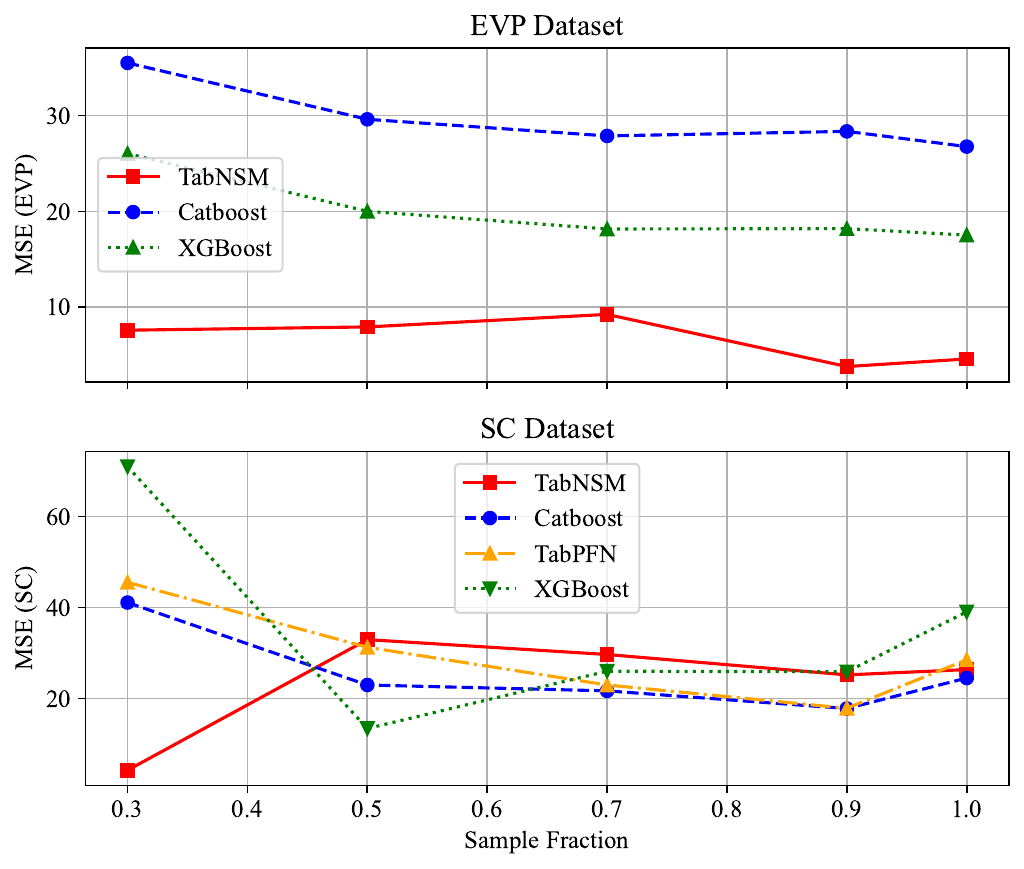}
    \caption{MSE versus training sample fraction on the TO and SC datasets.}
    \label{fig:sample}
  \end{subfigure}
  \caption{\textbf{Sensitivity to feature dimensionality and sample size.}
  The top panel evaluates robustness under varying numbers of retained input features, while the bottom panel evaluates robustness under varying fractions of training samples. Separate y-axes are used to accommodate differences in error scale.}
  \label{fig:dimension_sample}
\end{figure}

\subsection{Sensitivity to GridLoss Hyperparameters} \label{appendix:Gridloss}
We analyze the sensitivity of GridLoss to the temperature parameter $\tau$ and the blending coefficient $\alpha$ in the combined objective:
\[
\mathcal{L} = \alpha \mathcal{L}_{\text{Huber}} + (1-\alpha)\mathcal{L}_{\text{Grid}} .
\]
Figure~\ref{fig:grid_tau} reports RMSE as a function of $\tau$ across several datasets (CP, TO, and SCAN-2, abbreviated as SC). Performance varies smoothly over multiple orders of magnitude, indicating that GridLoss is not overly sensitive to the choice of $\tau$ within a broad range. In the main experiments, we use a fixed default setting for both $\tau$ and $\alpha$, while tuning the number of grid bins during training based on dataset-specific target statistics. The effects of varying $\tau$ and $\alpha$ are analyzed separately below.

\subsection{Sensitivity of \texorpdfstring{$\tau$}{tau} to Target Distribution Properties}
Table~\ref{tab:tau_behavior} summarizes the relationship between target distribution characteristics and the observed sensitivity of RMSE to the regularization parameter $\tau$. We find that the optimal choice of $\tau$ is primarily governed by the concentration of the bulk of the target distribution, as quantified by the ratio $\mathrm{IQR}/\mathrm{Std}$, rather than by tail heaviness alone.

For CP, the target distribution exhibits a moderately dispersed bulk with heavy tails. Increasing $\tau$ consistently improves robustness, resulting in a monotonic decrease in RMSE. In contrast, TO has a tightly concentrated bulk with rare but extreme outliers, favoring small $\tau$ values that preserve precision on the dominant mass while avoiding excessive smoothing. Finally, SCAN-2 is characterized by a degenerate bulk (IQR $=0$) and extremely rare but large outliers, leading to a highly non-monotonic RMSE--$\tau$ relationship in which very small $\tau$ values perform best.

These observations suggest that $\tau$ should be adapted to the effective bulk structure of the target distribution. Relying solely on tail statistics such as skewness or kurtosis is insufficient for selecting an appropriate regularization strength.

\begin{table}[tbh]
\centering
\caption{Target distribution statistics and their relationship to $\tau$ sensitivity.}
\label{tab:tau_behavior}
\begin{tabular}{lcccccl}
\toprule
Dataset & Std & IQR & IQR/Std & Tail Type & Preferred $\tau$ & RMSE--$\tau$ Behavior \\
\midrule
CP     & 18.40 & 13.00 & 0.71 & Heavy-tailed     & Large      & Monotonic decrease \\
TO     & 0.03  & 0.02  & 0.67 & Rare extremes    & Small      & Sharp peak \\
SCAN-2 & 4.99  & 0.00  & 0.00 & Degenerate bulk  & Very small & Highly non-monotonic \\
\bottomrule
\end{tabular}
\end{table}

\begin{figure}[tbh]
  \centering
  \resizebox{0.7\textwidth}{!}{
  \includegraphics[width=\linewidth]{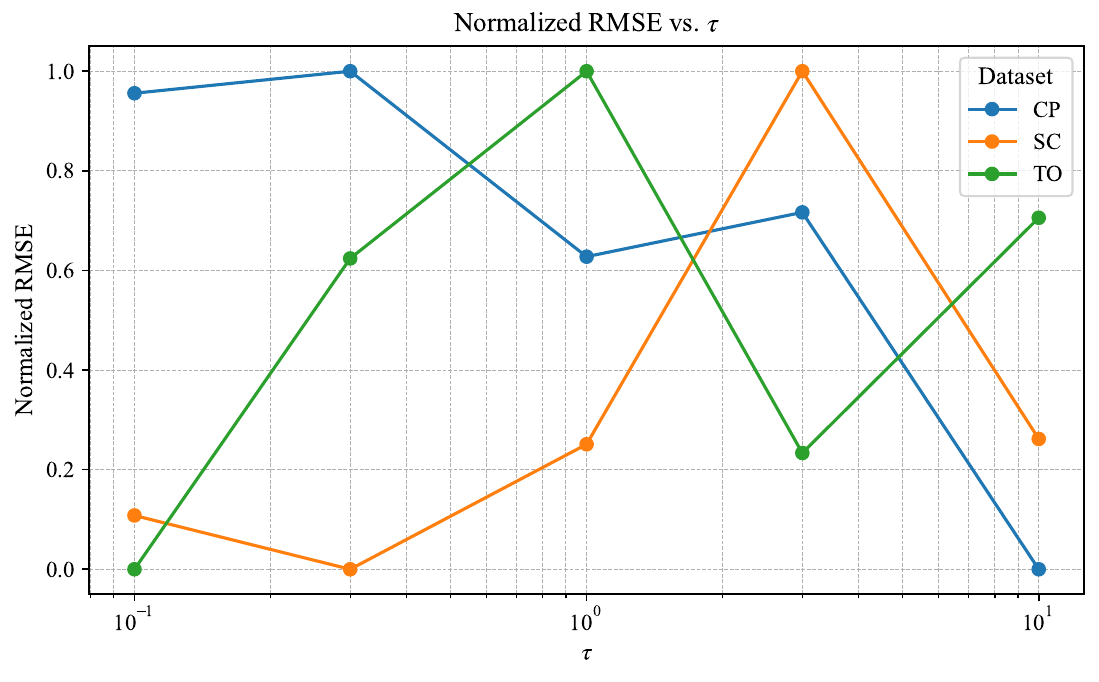}
  }
  \caption{Normalized RMSE as a function of $\tau$ across datasets.}
  \label{fig:grid_tau}
\end{figure}

\subsection{Effect of the Blending Coefficient \texorpdfstring{$\alpha$}{alpha}}
We next study the impact of the blending coefficient $\alpha$, which controls the trade-off between the Huber loss and the grid-based loss. As shown in Figure~\ref{fig:grid_alpha}, relying on either loss alone ($\alpha=0$ or $\alpha=1$) leads to degraded performance. In contrast, intermediate values of $\alpha$ consistently achieve lower normalized RMSE across all datasets, with optimal performance observed for $\alpha \in [0.4, 0.6]$. These results indicate that the two loss terms are complementary and that balancing robustness with structural regularization is critical for optimal performance.
\begin{figure}[tbh]
  \centering
  \resizebox{0.7\textwidth}{!}{
  \includegraphics[width=\linewidth]{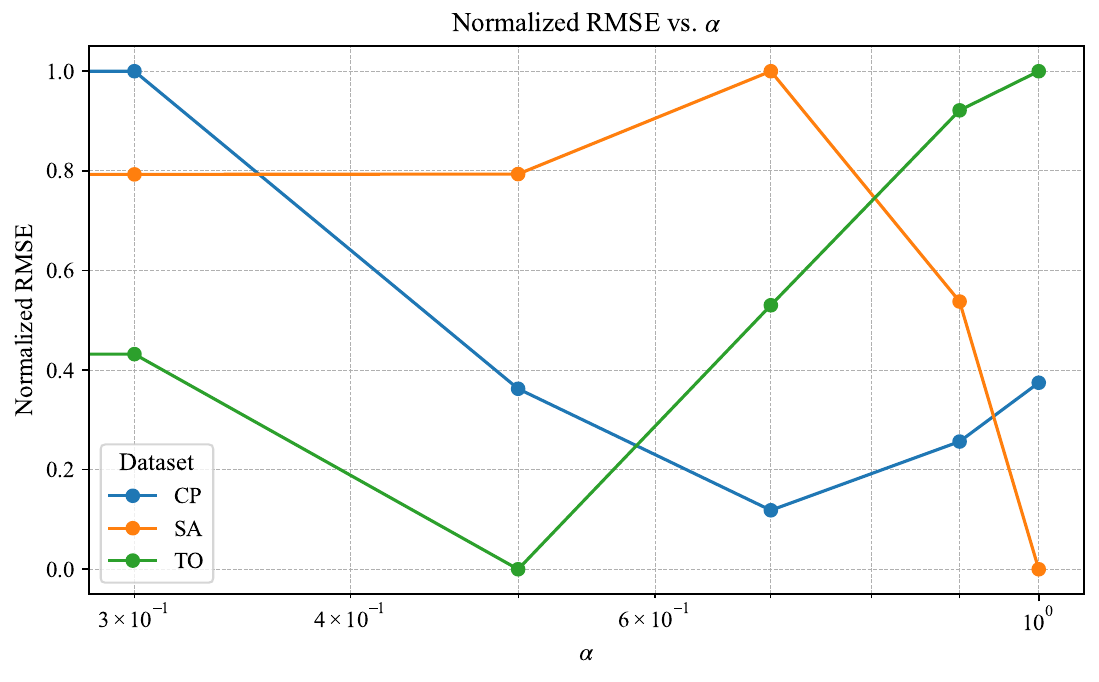}
  }
  \caption{Normalized RMSE as a function of $\alpha$ across datasets.}
  \label{fig:grid_alpha}
\end{figure}

\subsection{Component-wise Ablation of Sparse Attention Parameters}
\label{App:CompL}

We conduct a component-wise ablation study to assess the contribution of each sparse attention hyperparameter to model performance (Figure~\ref{fig:all_steps}). In each experiment, a single hyperparameter is varied while all others are held fixed, and performance is evaluated using average validation RMSE across four validation sets. Lower RMSE indicates better performance. For the \textit{compress block size}, a value of 10 achieves the best average validation RMSE, with performance degrading for both smaller and larger values. The \textit{number of selected blocks} performs best at 3, while deviations in either direction reduce effectiveness. Similarly, a \textit{selection block size} of 10 yields the lowest validation RMSE among the tested values. Finally, a \textit{sliding window size} of 6 provides the best performance, suggesting the importance of selecting an appropriate local context size. Based on these findings, we restrict the hyperparameter search space in the first-stage optimization to improve search efficiency. Nevertheless, joint hyperparameter optimization remains necessary to achieve optimal performance across datasets.

\begin{figure*}[tbh]
  \centering
  \includegraphics[width=0.48\textwidth]{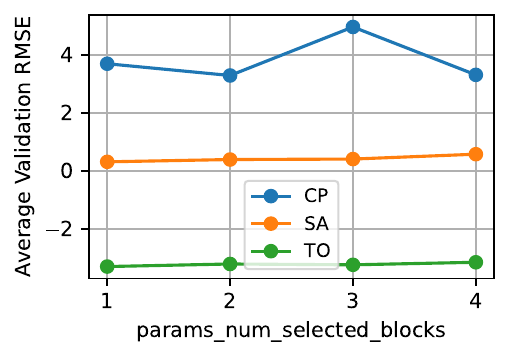}\hfill
  \includegraphics[width=0.48\textwidth]{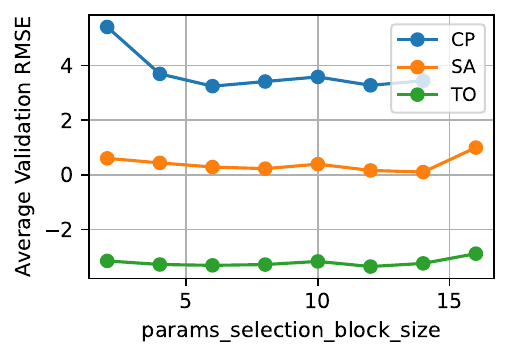}

  \vspace{1ex}

  \includegraphics[width=0.48\textwidth]{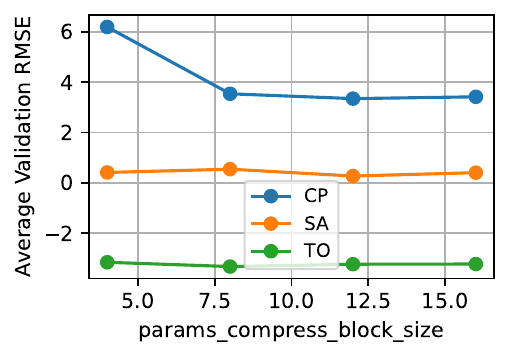}\hfill
  \includegraphics[width=0.48\textwidth]{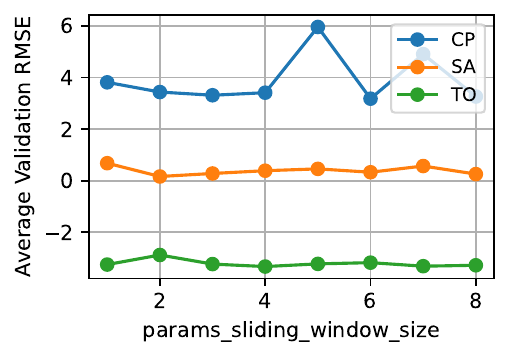}

  \caption{Component-wise ablation of sparse attention parameters. Each plot shows the effect of varying one parameter: \textit{number\_of\_selected\_blocks} (top-left), \textit{selection\_block\_size} (top-right), \textit{compress\_block\_size} (bottom-left), and \textit{sliding\_window\_size} (bottom-right) on the average validation RMSE across four datasets.}
  \label{fig:all_steps}
\end{figure*}

\end{document}